\documentclass[10pt,twocolumn,letterpaper]{article}
\usepackage[T1]{fontenc}

\usepackage[pagenumbers]{cvpr} 

\usepackage{xargs}
\usepackage{fancybox}

\newcommand{\aka}{\emph{a.k.a.}\xspace}
\newcommand{\iif}{\emph{iif}\xspace}

\newcommand{\first}{1\textsuperscript{st}\xspace}
\newcommand{\second}{2\textsuperscript{nd}\xspace}
\newcommand{\third}{3\textsuperscript{rd}\xspace}

\newcommand{\mysection}[1]{\vspace{2pt}\noindent\textbf{#1}}

\usepackage{multicol}
\usepackage{longtable}
\usepackage{rotating}
\usepackage{multirow}
\usepackage{nicefrac}
\usepackage{amsthm}
\usepackage{xargs}[2008/03/08]
\usepackage{pifont}
\usepackage{tikz}
\usepackage{epsdice}

\usepackage{thmtools}
\newtheorem{theorem}{Theorem}
\newtheorem{corollary}{Corollary}

\newtheorem{property}{Property}
\newtheorem{example}{Example}
\newtheorem{lemma}{Lemma}
\newtheorem{experiment}{Random Experiment}

\usepackage[dvipsnames]{xcolor}
\usepackage{wasysym}

\usepackage{titletoc}

\definecolor{cvprblue}{rgb}{0.21,0.49,0.74}
\usepackage[pagebackref,breaklinks,colorlinks,allcolors=cvprblue]{hyperref}

\def\paperID{5644} 
\def\confName{CVPR}
\def\confYear{2025}

\title{Foundations of the Theory of Performance-Based Ranking}

\author{S\'ebastien Pi\'erard, Ana\"is Halin, Anthony Cioppa, Adrien Deli\`ege, and Marc Van Droogenbroeck\\
Montefiore Institute, University of Li\`ege, Li\`ege, Belgium\\
{\tt\small \{S.Pierard,Anais.Halin,Anthony.Cioppa,Adrien.Deliege,M.VanDroogenbroeck\}@uliege.be}
}

\begin{document}

\newcommand{\paperA}{paper~A~\cite{Pierard2024Foundations-arxiv}\xspace}
\newcommand{\paperB}{paper~B~\cite{Pierard2024TheTile-arxiv}\xspace}
\newcommand{\paperC}{paper~C~\cite{Halin2024AHitchhikers-arxiv}\xspace}
\newcommand{\PaperA}{Paper~A~\cite{Pierard2024Foundations-arxiv}\xspace}
\newcommand{\PaperB}{Paper~B~\cite{Pierard2024TheTile-arxiv}\xspace}
\newcommand{\PaperC}{Paper~C~\cite{Halin2024AHitchhikers-arxiv}\xspace}

\global\long\def\sampleSpace{\Omega}%
\global\long\def\aSample{\omega}%
\global\long\def\eventSpace{\Sigma}%
\global\long\def\anEvent{E}%
\global\long\def\measurableSpace{(\sampleSpace,\eventSpace)}%
\global\long\def\expectedValueSymbol{\mathbf{E}}%

\global\long\def\aPerformance{P}%
\global\long\def\allPerformances{\mathbb{\aPerformance}_{\measurableSpace}}%
\global\long\def\aSetOfPerformances{\Pi}%
\global\long\def\randVarSatisfaction{S}%
\global\long\def\aScore{X}%
\global\long\def\allScores{\mathbb{\aScore}_{\measurableSpace}}%
\newcommandx\domainOfScore[1][usedefault, addprefix=\global, 1=\aScore]{\mathrm{dom}(#1)}%
\global\long\def\evaluation{\mathrm{eval}}%
\global\long\def\opFilter{\mathrm{filter}_\randVarImportance}%
\global\long\def\opNoSkill{\mathrm{no\text{\textendash{}}skill}}%
\global\long\def\opPriorShift{\mathrm{shift}_{\pi\rightarrow\pi'}}%
\global\long\def\opChangePredictedClass{\mathrm{change}_{\randVarPredictedClass}}%
\global\long\def\opChangeGroundtruthClass{\mathrm{change}_{\randVarGroundtruthClass}}%
\global\long\def\opSwapGroundtruthAndPredictedClasses{\mathrm{swap}_{\randVarGroundtruthClass\leftrightarrow\randVarPredictedClass}}%
\global\long\def\opSwapClasses{\mathrm{swap}_{\classNeg\leftrightarrow\classPos}}%
\global\long\def\allWorstPerformances{\frownie}
\global\long\def\allBestPerformances{\smiley}

\global\long\def\randVarGroundtruthClass{Y}%
\global\long\def\randVarPredictedClass{\hat{Y}}%
\global\long\def\allClasses{\mathbb{C}}%
\global\long\def\aClass{c}%
\global\long\def\classNeg{c_-}%
\global\long\def\classPos{c_+}%
\global\long\def\sampleTN{tn}%
\global\long\def\sampleFP{fp}%
\global\long\def\sampleFN{fn}%
\global\long\def\sampleTP{tp}%
\global\long\def\eventTN{\{\sampleTN\}}%
\global\long\def\eventFP{\{\sampleFP\}}%
\global\long\def\eventFN{\{\sampleFN\}}%
\global\long\def\eventTP{\{\sampleTP\}}%
\global\long\def\scorePTN{PTN}%
\global\long\def\scorePFP{PFP}%
\global\long\def\scorePFN{PFN}%
\global\long\def\scorePTP{PTP}%
\global\long\def\scoreAccuracy{A}%
\global\long\def\scoreExpectedSatisfaction{\aScore_{\randVarSatisfaction}}%
\global\long\def\scoreTNR{TNR}%
\global\long\def\scoreFPR{FPR}%
\global\long\def\scoreTPR{TPR}%
\global\long\def\scoreFNR{FNR}%
\global\long\def\scoreNPV{NPV}%
\global\long\def\scoreFOR{FOR}%
\global\long\def\scorePPV{PPV}%
\global\long\def\scorePrecision{\scorePPV}%
\global\long\def\scoreFDR{FDR}%
\global\long\def\scoreJaccardNeg{J_-}%
\global\long\def\scoreJaccardPos{J_+}%
\global\long\def\scoreCohenKappa{\kappa}%
\global\long\def\scoreScottPi{\pi}%
\global\long\def\scoreFleissKappa{\kappa}%
\global\long\def\scoreBalancedAccuracy{BA}%
\global\long\def\scoreWeightedAccuracy{WA}%
\global\long\def\scoreYoudenJ{J_Y}
\global\long\def\scorePLR{PLR}%
\global\long\def\scoreNLR{NLR}%
\global\long\def\scoreOR{OR}%
\global\long\def\scoreSNPV{SNPV}%
\global\long\def\scoreSPPV{SPPV}%
\global\long\def\scoreACP{ACP}%
\global\long\def\scoreFOne{F_{1}}%
\newcommandx\scoreFBeta[1][usedefault, addprefix=\global, 1=\beta]{F_{#1}}%
\global\long\def\scoreFOne{\scoreFBeta[1]}%
\global\long\def\priorpos{\pi_+}%
\global\long\def\priorneg{\pi_-}%
\global\long\def\scoreBiasIndex{BI}%
\global\long\def\ratepos{\tau_+}%
\global\long\def\rateneg{\tau_-}%
\global\long\def\scoreACP{ACP}%
\global\long\def\scorePFour{P_4}%
\global\long\def\normalizedConfusionMatrix{C}%
\global\long\def\scoreConfusionMatrixDeterminant{|\normalizedConfusionMatrix|}

\global\long\def\allEntities{\mathbb{E}}%
\global\long\def\entitiesToRank{\mathbb{E}}%
\global\long\def\anEntity{\epsilon}%
\global\long\def\randVarImportance{I}%
\global\long\def\randVarCanonicalImportance{\randVarImportance_{a,b}}
\global\long\def\canonicalRankingScore{\rankingScore[\randVarCanonicalImportance]}
\newcommandx\rankingScore[1][usedefault, addprefix=\global, 1=\randVarImportance]{R_{#1}}%
\global\long\def\canonicalRankingScore{\rankingScore[\randVarImportance_{a,b}]}
\global\long\def\scoreVUT{VUT}%
\global\long\def\tileCurvePriors{\gamma_\pi}
\global\long\def\tileCurveRates{\gamma_\tau}
\global\long\def\relWorseOrEquivalent{\lesssim}%
\global\long\def\relBetterOrEquivalent{\gtrsim}%
\global\long\def\relEquivalent{\sim}%
\global\long\def\relBetter{>}%
\global\long\def\relWorse{<}%
\global\long\def\relIncomparable{\not\lesseqqgtr}%
\global\long\def\rank{\mathrm{rank}_\entitiesToRank}%
\global\long\def\ordering{\relWorseOrEquivalent}%
\global\long\def\invertedOrdering{\relBetterOrEquivalent}%

\global\long\def\LScityscapes{\ding{171}}
\global\long\def\LSade{\ding{170}}
\global\long\def\LSvoc{\ding{169}}
\global\long\def\LScoco{\ding{168}}

\global\long\def\indicatorSymbol{\mathbf{1}}
\global\long\def\realNumbers{\mathbb{R}}%
\global\long\def\aRelation{\mathcal{R}}%
\global\long\def\achievableByCombinations{\Phi}%
\global\long\def\allConvexCombinations{\mathrm{conv}}%
\newcommand{\indep}{\perp \!\!\! \perp}


\global\long\def\cityscapes{\LScityscapes{}~Cityscapes}
\global\long\def\ade{\LSade{}~ADE20K}
\global\long\def\voc{\LSvoc{}~Pascal VOC 2012}
\global\long\def\coco{\LScoco{}~COCO-Stuff 164k}

\newcommand{\MethodDesigner}{Bernadette}
\newcommand{\Benchmarker}{Leonard}
\newcommand{\AppDeveloper}{Howard}
\newcommand{\TheoreticalAnalyst}{Sheldon}

\newcommand{\tile}{Tile\xspace}
\newcommand{\tiles}{Tiles\xspace}
\newcommand{\valueTile}{Value Tile\xspace}
\newcommand{\baselineTile}{Baseline Value Tile\xspace}
\newcommand{\SOTATile}{State-of-the-Art Value Tile\xspace}
\newcommand{\noSkillTile}{No-Skill Tile\xspace}
\newcommand{\skillTile}{Relative-Skill Tile\xspace}
\newcommand{\correlationTile}{Correlation Tile\xspace}
\newcommand{\rankingTile}{Ranking Tile\xspace}
\newcommand{\entityTile}{Entity Tile\xspace}

\global\long\def\aNonSkilledPerformance{\aPerformance_{\indep}}
\global\long\def\allNonSkilledPerformances{\mathbb{\aPerformance}^{\randVarGroundtruthClass\indep\randVarPredictedClass}_{\measurableSpace}}%

\global\long\def\allPriorFixedPerformances{\mathbb{\aPerformance}^{\priorpos}_{\measurableSpace}}%

\newcommand{\comma}{\,,}
\newcommand{\point}{\,.}

\newcommandx\unconditionalProbabilisticScore[1]{\aScore_{#1}^{U}}%
\global\long\def\formulaPTN{\unconditionalProbabilisticScore{\eventTN}}%
\global\long\def\formulaPFP{\unconditionalProbabilisticScore{\eventFP}}%
\global\long\def\formulaPFN{\unconditionalProbabilisticScore{\eventFN}}%
\global\long\def\formulaPTP{\unconditionalProbabilisticScore{\eventTP}}%
\global\long\def\formulapriorneg{\unconditionalProbabilisticScore{\{\sampleTN,\sampleFP\}}}%
\global\long\def\formulapriorpos{\unconditionalProbabilisticScore{\{\sampleFN,\sampleTP\}}}%
\global\long\def\formularateneg{\unconditionalProbabilisticScore{\{\sampleTN,\sampleFN\}}}%
\global\long\def\formularatepos{\unconditionalProbabilisticScore{\{\sampleFP,\sampleTP\}}}%
\global\long\def\formulaAccuracy{\unconditionalProbabilisticScore{\{\sampleTN,\sampleTP\}}}%

\newcommandx\conditionalProbabilisticScore[2]{\aScore_{#1 \vert #2}^{C}}%
\global\long\def\formulaTNR{\conditionalProbabilisticScore{\{\sampleTN\}}{\{\sampleTN,\sampleFP\}}}%
\global\long\def\formulaTPR{\conditionalProbabilisticScore{\{\sampleTP\}}{\{\sampleFN,\sampleTP\}}}%
\global\long\def\formulaNPV{\conditionalProbabilisticScore{\{\sampleTN\}}{\{\sampleTN,\sampleFN\}}}%
\global\long\def\formulaPPV{\conditionalProbabilisticScore{\{\sampleTP\}}{\{\sampleFP,\sampleTP\}}}%
\global\long\def\formulaJaccardNeg{\conditionalProbabilisticScore{\{\sampleTN\}}{\{\sampleTN,\sampleFP,\sampleFN\}}}%
\global\long\def\formulaJaccardPos{\conditionalProbabilisticScore{\{\sampleTP\}}{\{\sampleFP,\sampleFN,\sampleTP\}}}%

\global\long\def\scoreBennettS{S}


\renewcommand{\paperA}{\textcolor{red}{XXXXXXXX}\xspace}
\renewcommand{\paperB}{\cite{Pierard2024TheTile-arxiv}\xspace}
\renewcommand{\paperC}{\cite{Halin2024AHitchhikers-arxiv}\xspace}
\renewcommand{\PaperA}{\textcolor{red}{XXXXXXXX}\xspace}
\renewcommand{\PaperB}{Paper~\cite{Pierard2024TheTile-arxiv}\xspace}
\renewcommand{\PaperC}{Paper~\cite{Halin2024AHitchhikers-arxiv}\xspace}
\global\long\def\scoreExpectedSatisfaction{\expectedValueScore{\randVarSatisfaction}}%

\maketitle

\begin{abstract}

Ranking entities such as algorithms, devices, methods, or models based on their performances, while accounting for application-specific preferences, is a challenge. To address this challenge, we establish the foundations of a universal theory for performance-based ranking. First, we introduce a rigorous framework built on top of both the probability and order theories. Our new framework encompasses the elements necessary to (1) manipulate performances as mathematical objects, (2) express which performances are worse than or equivalent to others, (3) model tasks through a variable called satisfaction, (4) consider properties of the evaluation, (5) define scores, and (6) specify application-specific preferences through a variable called importance. On top of this framework, we propose the first axiomatic definition of performance orderings and performance-based rankings. Then, we introduce a universal parametric family of scores, called \emph{ranking scores}, that can be used to establish rankings satisfying our axioms, while considering application-specific preferences. Finally, we show, in the case of two-class classification, that the family of ranking scores encompasses well-known performance scores, including the accuracy, the true positive rate (recall, sensitivity), the true negative rate (specificity), the positive predictive value (precision), and $\scoreFBeta[1]$. However, we also show that some other scores commonly used to compare classifiers are unsuitable to derive performance orderings satisfying the axioms.

\end{abstract}

\begin{figure}[t]
\begin{centering}
\resizebox{1\linewidth}{!}{
\includegraphics[width=\linewidth]{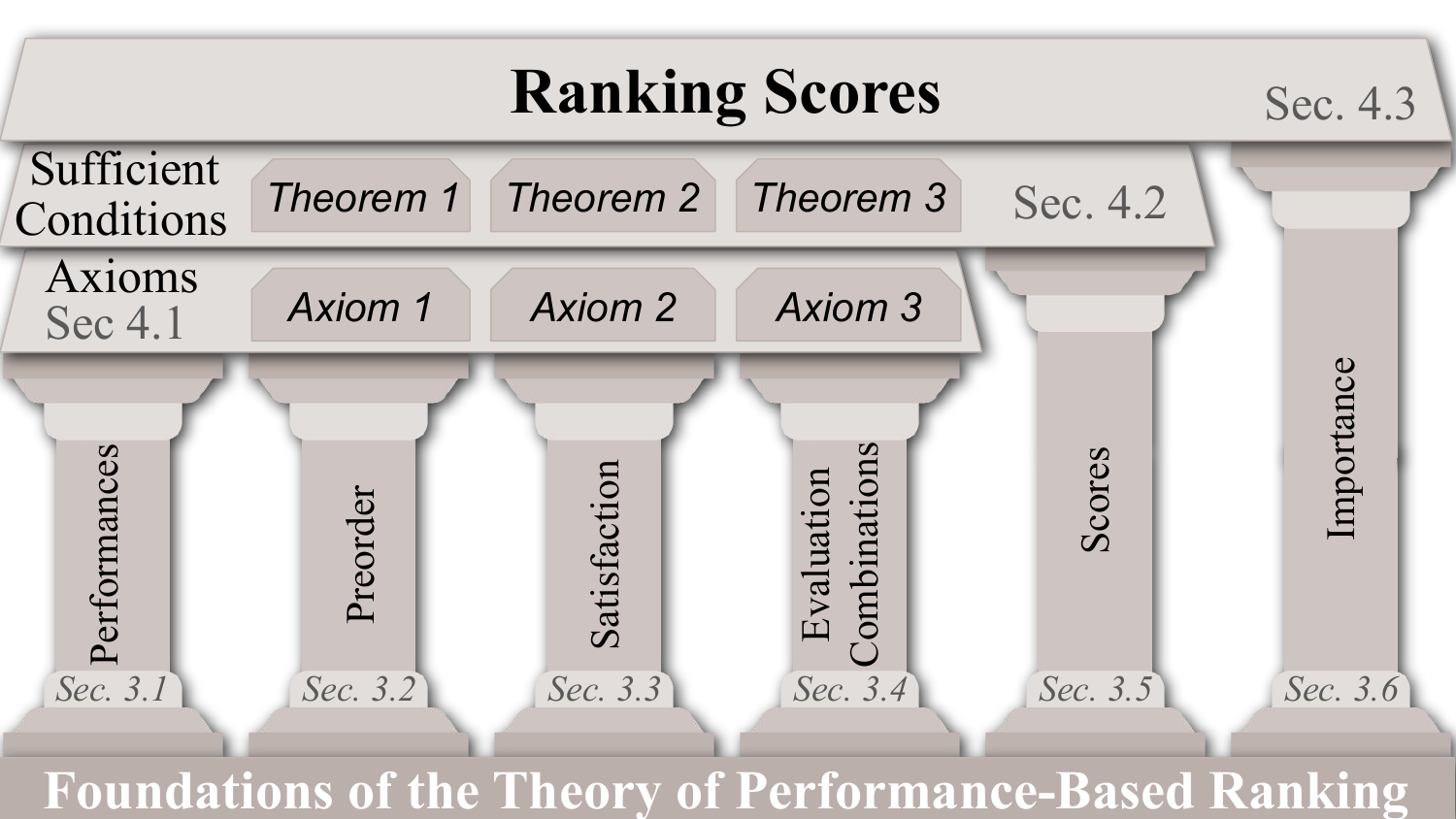}}
\par\end{centering}
\caption{This work establishes the foundations of the theory of performance-based ranking. We do this in two steps. First, we introduce a new mathematical framework with $6$ main elements, as depicted here by the pillars. Second, we build on top of it: (1) a set of three axioms for the ordering of performances and for the performance-based ranking of entities, (2) sufficient conditions for them when the performance ordering is induced by a score, and (3) a family of scores, named \emph{ranking scores} that consider the application-specific preferences. This theory is universal in the sense that it is applicable to any task.
\label{fig:graphical_abstract}}
\end{figure}

\section{Introduction}
\label{sec:intro}


Every day, millions of people are faced with choices to make. Often, these choices are between entities (\eg, algorithms, devices, methods, models, options, procedures, solutions, strategies, \etc) 
considered to be interchangeable, although not necessarily equivalent in terms of performance. One of the main difficulties arises from the uncertainty that people have regarding the use that will be made of the entity to choose. A widespread approach to objectifying these choices is to (1) perform an evaluation to determine (\ie, assume, calculate, estimate, predict, \etc) a \emph{performance}, encompassing the necessary uncertainty, for each of these entities; (2) choose a way of comparing these performances with each other; and (3) assume that an entity is preferable to others if it has the best performance.
A more general problem is to establish an order of preference between the entities: this is the \emph{performance-based ranking}.

The approach of performance-based ranking is common in many fields and has proved its usefulness, especially in scientific communities that organize themselves around competitions~\cite{Cioppa2024SoccerNet2023Challenge, Goyette2014ANovel, Kristan2015TheVisual, Kristan2024TheSecond} for the development of algorithms for specific tasks. Nevertheless, several studies \cite{MaierHein2018WhyRankings,Nguyen2023HowTrustworthy} have  alerted the scientific community about the ranking methodology used in these competitions.

A critical analysis~\cite{MaierHein2018WhyRankings} of common practices for 150 biomedical image analysis challenges reveals that the scores used are justified in only 23\% of the cases, and that the rank computation method is reported in only 36\% of the cases. Moreover, there are at least 10 different methods for determining the rank of an algorithm based on multiple scores. The properties of these methods are largely unknown.

There remains an obvious lack of theoretical foundations for the performance-based ranking. From our perspective, there is a common and detrimental confusion between the concept of performance and the numerical scores (also called metrics, measures, indicators, criteria, factors, and indices~\cite{Texel2013Measure,Canbek2017Binary}). Moreover, the way of comparing performances is often chosen by intuition~\cite{Nguyen2023HowTrustworthy} or by imitation of what has been chosen in the previously published related works. This inevitably leads to a drastic loss of diversity in the rankings that can be presented. From our perspective, considering the diversity of possible rankings is healthy, and also desirable for those who want to subsequently choose an entity based on their own application-specific preferences.

The lack of theoretical foundations is profound. Loosely speaking, the \emph{performance} of a given entity can be defined as the information necessary to determine the degree of \emph{satisfaction} one has with it, relatively to application-specific preferences tuned by the relative \emph{importance} given to the various cases that can occur when the entity is used. Although such a definition sheds interesting light on the subject, it does not specify the mathematical nature of the performance, the space in which it is, or the operations that are permitted, particularly those that enable the various performances to be compared with one another.


The aim of this paper is to look at the performance-based ranking from a broader angle. As shown in \cref{fig:graphical_abstract}, we establish mathematically rigorous theoretical foundations that can be applied to a wide range of problems.  Throughout the paper, we exemplify our theory with the task of two-class classification. More tasks are detailed in \cref{sec:catalog-of-problems} about: multi-class classification with links to micro- and macro-averaging; regression with links to the mean squared error and the mean absolute error; information retrieval; detection with links to the intersection over-union and the F-score; clustering with a link to the Fowlkes-Mallows index; ranking with a link to Kendall's $\tau$.


\mysection{Contributions.} Our contributions are threefold. 
(1) First, in \cref{sec:mathematical-framework}, we present a new mathematical framework that encompasses the evaluation of entities, the performances, the space in which they are, the notions of \emph{satisfaction} and \emph{importance}, the scores that characterize numerically the performances, and some  operations permitted on performances such as those for combining several performances and those for comparing performances with one another (\emph{is equivalent to}, \emph{is worse than}, \emph{is better than}, \emph{is incomparable with}, \etc).
(2) Second, we introduce the foundations of the theory of performance-based ranking in \cref{sec:ranking-theory}. We innovate with an \emph{axiomatic definition for both the ordering of performances and the ranking of entities based on their performances}. We also present a new family of scores, named \emph{ranking scores}, that can be used to induce performance orderings that satisfy our axioms. These scores are parameterized by application-specific preferences, and are universal, \ie, applicable to any task.
(3) Third, in \cref{sec:two-class-crisp-classification}, we study the particularization of our theory of performance-based ranking to the popular case of the two-class crisp classification.
\section{Related Work}

Mathematical foundations for our work can be found in the probability theory, in the order theory, and in statistics. The probability theory~\cite{Kolmogorov1933Grundbegriffe,Kolmogorov1950Foundations} provides the tools needed to consider the uncertainty that one has about how an entity will be used. However, to be rigorous, probability measures cannot be used without defining a measurable space. Surprisingly, after reviewing hundreds of papers that use the notion of performance, mainly in the fields of computer vision, medicine and physics, we found none that explicitly gives such a space and explains how to express, based on it, what the performances are and what operations are allowed on them. The order theory~\cite{Davey2002Introduction,Gratzer2003General} provides the basis for defining and manipulating homogeneous binary relations such as \emph{is equivalent to}, \emph{is worse than}, \emph{is better than}, and \emph{is incomparable with} that underpin rankings. Statistics provide tools to compare rankings through rank correlations, in particular Kendall's $\tau$~\cite{Kendall1938ANewMeasure} and Spearman's $\rho$~\cite{Spearman1904Theproof}. 

In a recent attempt to formalize the notion of ranking, for basic vision tasks, Nguyen \etal~\cite{Nguyen2023HowTrustworthy} proposed to impose three properties for ranking: (1) reliability (``a small change in parameter values, should not result in a drastic change in rankings''), (2) meaningfulness (evaluated by humans), and (3) mathematical consistency (``use scores that satisfy certain properties''). The mathematical framework introduced in this paper helps to clarify these three requirements. (1) Regarding the reliability, which is also a matter of concern in~\cite{MaierHein2018WhyRankings}, we argue that one should distinguish between two types of parameters: those involved in the evaluation (\ie the step in which the performance of an entity is determined) and those involved in the ranking of the entities based on their previously determined performances. The case actually discussed in~\cite{Nguyen2023HowTrustworthy} is of the first type. The second type of parameters can be useful to adapt the ranking to application-specific preferences: the \emph{importance}. (2) Regarding the meaningfulness, in the absence of analytical means to ensure the meaningfulness of scores, Nguyen \etal~\cite{Nguyen2023HowTrustworthy} suggest testing the ranking procedure on sanity tests with pre-determined desired rankings. 
In contrast, we propose to start by modeling the task through a variable called \emph{satisfaction} and then to derive meaningful ranking scores, by construction. (3) Concerning the mathematical consistency, we show that it is possible to impose it based on axioms, at a level just below the scores (see \cref{fig:graphical_abstract}).

\section{Mathematical Framework}
\label{sec:mathematical-framework}

We now present the mathematical framework in which the theory of performance-based rankings will be established in \cref{sec:ranking-theory}. 

To the best of our knowledge, it is the first time that a rigorous mathematical framework is conceived for the universal comparison of performances. All mathematical symbols used in this paper are defined where they first appear. For convenience, we also provide a list of them in~\cref{sec:list_of_symbols}. Our framework involves six components that correspond to the pillars depicted in \cref{fig:graphical_abstract}.

\mysection{Performance to address uncertainty.} First, we introduce the notion of \emph{performance}, which is our main  object of interest. To anchor it on a solid mathematical ground, we leverage 
probability theory to benefit from its established expressivity and rigor, and thus define performances as probability measures. Probability theory is indeed the ideal framework for studying uncertainty and randomness, which naturally pertain to performances, hence our choice. 

\mysection{The essence of performance.} Then, we define the three components that really differentiate performances from ordinary probability measures: performances should be comparable/rankable (\eg, we need notions of \emph{better} and \emph{worse} performances) through a \emph{preorder} $\relWorseOrEquivalent$; performances should be related to a task, that we model by a random variable called \emph{satisfaction} $\randVarSatisfaction$; and performances should be related to entities (algorithms, devices, \etc) through an \emph{evaluation}, that we model by a function $\achievableByCombinations$. We show that there exist compatibility conditions to be met between $\relWorseOrEquivalent$ and $\randVarSatisfaction$ as well as between $\relWorseOrEquivalent$ and $\achievableByCombinations$, which lead us to formalize the \emph{axioms} at the basis of our theory in~\cref{sec:ranking-theory}. 

\mysection{Performance in practical applications.} Finally, we define two components that connect our theory to practical applications: the \emph{scores} $\aScore$, which are functions mapping performances to numerical values (\eg accuracy, error rate,\etc); and the \emph{importance} $\randVarImportance$, a random variable that encodes applicative preferences about the possible outcomes of the process. Later in~\cref{sec:ranking-theory}, we show how these two components can be defined to satisfy sufficient conditions (through three theorems) to fulfill our three axioms, thus further serving as basis for our performance-based ranking theory and the definition of new universal ranking scores.

\subsection{The Performance $\aPerformance$ as a Probability Measure}
\label{sec:mathematical-framework-performances}

For us, a performance is not a real number or a collection of them, as sometimes assumed in the literature. Instead, we choose to design our mathematical framework specifically to compare performances having a probabilistic meaning. Thus, we ground our framework in probability theory~\cite{Kolmogorov1933Grundbegriffe,Kolmogorov1950Foundations} and consider that \emph{performances} $\aPerformance$ are probability measures. For two performances to be comparable, they should be defined on a common measurable space $\measurableSpace$, where $\sampleSpace$ is the sample space, or universe, and $\eventSpace$ is a $\sigma$-algebra on $\sampleSpace$ called event space. Without loss of generality, when $\sampleSpace$ is finite, one can choose $\eventSpace=2^\sampleSpace$. We note $\allPerformances$ the set of probability measures on $\measurableSpace$.

\begin{example}
    For the popular case of two-class crisp classification, several choices can be made for $\sampleSpace$. In the simplest setting, this set contains two elements interpreted as ``correct result'' and ``incorrect result''. In this case, the performance analysis can only be based on the proportion of correct results, \ie, the accuracy. In another setting, one can choose to have three elements in $\sampleSpace$, making the distinction between the two types of incorrect results, namely \emph{false positive} (\aka type I error) and \emph{false negative} (\aka type II error). Finally, one could prefer having four elements in $\sampleSpace$, one for each pair of ground-truth and predicted classes (in the frequency-based approach, this corresponds to the normalized contingency table or confusion matrix). This enlarges the flexibility in the analysis of performances.
\end{example}

\subsection{Ordering Performances with a Preorder $\relWorseOrEquivalent$}
\label{sec:mathematical-framework-ordering}


Our framework is not only grounded in  probability theory, as explained above, but also in order theory~\cite{Davey2002Introduction,Gratzer2003General}. We aim at being able to decide if a performance is, \eg, worse than, equivalent, or better than another one. Mathematically, these comparisons are done thanks to binary homogeneous relations on $\allPerformances$. Indeed, all binary relations used to compare performances should be coherent, \ie, they should all correspond to a common \emph{performance ordering}.  
Following this path, in this paper, the binary relations $\relEquivalent$, $\relBetter$, $\relWorse$, and $\relIncomparable$ on $\allPerformances$ will be implicitly considered as derived from a common binary relation $\relWorseOrEquivalent$ on $\allPerformances$ as follows:
\begin{itemize}
    \item $\aPerformance_1\relEquivalent\aPerformance_2$ if and only if (\iif) $\aPerformance_1\relWorseOrEquivalent\aPerformance_2 \wedge \aPerformance_2\relWorseOrEquivalent\aPerformance_1$;
    \item $\aPerformance_1\relBetter\aPerformance_2$ \iif  $\aPerformance_1\not\relWorseOrEquivalent\aPerformance_2 \wedge \aPerformance_2\relWorseOrEquivalent\aPerformance_1$;
    \item $\aPerformance_1\relWorse\aPerformance_2$ \iif $\aPerformance_1\relWorseOrEquivalent\aPerformance_2 \wedge \aPerformance_2\not\relWorseOrEquivalent\aPerformance_1$;
    \item $\aPerformance_1\relIncomparable\aPerformance_2$ \iif $\aPerformance_1\not\relWorseOrEquivalent\aPerformance_2 \wedge \aPerformance_2\not\relWorseOrEquivalent\aPerformance_1$.
\end{itemize}
With such a construction, if $\relWorseOrEquivalent$ is a preorder (\ie, reflexive and transitive), then $\relEquivalent$ is an equivalence (\ie reflexive, transitive, and symmetric), $\relBetter$ and $\relWorse$ are converse strict partial orders (\ie, irreflexive, asymmetric, and transitive), and $\relIncomparable$ is irreflexive and symmetric. For the proof, see \cref{sec:Properties_of_orders}. These are the intuitively expected properties that justify to interpret $\relWorseOrEquivalent$ as \emph{worse or equivalent}, $\relBetter$ as \emph{better}, $\relWorse$ as \emph{worse}, and $\relIncomparable$ as \emph{incomparable}.

\addtocounter{example}{-1}
\begin{example}[continued]
    In the case of two-class classification, two spaces are commonly used to depict the performances as points, at fixed priors: the \emph{Receiver Operating Characteristic} (ROC) space and the \emph{Precision-Recall} (PR) space. On the one hand, ROC users will all intuitively agree that a performance $\aPerformance_1$ is better than, equivalent to, or worse than a performance $\aPerformance_2$ when both the values of $\scoreTNR$ (true negative rate) and the $\scoreTPR$ (true positive rate) of $\aPerformance_1$ are greater, equal, or smaller than those of $\aPerformance_2$, respectively. 
    Few of these users will risk deciding when one of the two scores is higher while the other is lower, as this depends on application-specific preferences. On the other hand, PR users will perform a similar intuitive reasoning based on the $\scoreTPR$ (recall) and the $\scorePPV$ (precision). A careful comparison reveals that ROC and PR users do not intuitively make the same decisions. This is surprising since, at fixed priors, both spaces show the same thing~\cite{Davis2006TheRelationship}. Our theory clarifies what are the suitable performance orderings.
\end{example}

\subsection{Modeling the Task as a Random Variable $\randVarSatisfaction$}
\label{sec:mathematical-framework-satisfaction}

\global\long\def\minSatisfaction{s_{min,\sampleSpace}}
\global\long\def\maxSatisfaction{s_{max,\sampleSpace}}

The random variable \emph{satisfaction}, $\randVarSatisfaction:\sampleSpace\rightarrow\realNumbers$, is task-specific. The user is responsible for assigning satisfaction values that are meaningful for the task at hand. We argue that a task is ill-defined when the satisfaction is not specified, either explicitly or implicitly. In this paper, we pose $\minSatisfaction=\min_{\aSample\in\sampleSpace} S(\aSample)$ and $\maxSatisfaction=\max_{\aSample\in\sampleSpace} S(\aSample)$.

\addtocounter{example}{-1}
\begin{example}[continued]
    In the case of two-class classification, we expect that most people will agree that (1) the satisfaction takes the same value for all samples corresponding to incorrect results (\eg $\randVarSatisfaction=0$), (2) the satisfaction takes the same value for all samples corresponding to correct results (\eg $\randVarSatisfaction=1$), and that (3) the satisfaction is strictly greater for correct results than for incorrect results.
\end{example}

\subsection{Modeling the Evaluation as a Function $\achievableByCombinations$}
\label{sec:mathematical-framework-combinations}

We now have all the elements necessary to introduce the notion of \emph{evaluation}. Let us denote, by $\allEntities$, the set of entities of interest. In our framework, the evaluation is modeled by function 
$\evaluation:\entitiesToRank\rightarrow\allPerformances:\anEntity\mapsto\evaluation(\anEntity)$. We find it convenient to think in terms of a (thought) random experiment involving an entity $\anEntity$ and having outcomes that allow to determine how satisfying the various realizations are, \eg, if the result is correct, how accurate it is, or how many resources are used. We define the performance $\aPerformance=\evaluation(\anEntity)$ of an entity $\anEntity$ as the distribution of these outcomes.

\addtocounter{example}{-1}
\begin{example}[continued]
    In the particular case of classification, a typical random experiment implicitly considered in the literature is in five steps. (1)~Draw a sample $s$ at random from a source. (2)~Apply the oracle on $s$ to obtain ground-truth class $y(s)$. (3)~Apply a descriptor on $s$ to obtain the features (\aka attributes) $x(s)$. (4)~Feed the classifier with $x(s)$ to obtain the predicted class $\hat{y}(x(s))$. (5)~Set the outcome of the experiment to the pair $(y,\hat{y})$.
\end{example}

Our framework can be further enriched by integrating some knowledge about the function $\evaluation$. By definition, a performance $\aPerformance$ is achievable when there exists an entity $\anEntity$ whose evaluation leads to it: $\exists\anEntity:\evaluation(\anEntity)=\aPerformance$. It often happens that, based solely on the performances of the evaluated entities, we can be sure that some other performances are also achievable, by combining and/or disrupting the entities we have at our disposal. To take this knowledge into account, we define the function $\achievableByCombinations:2^{\allPerformances}\rightarrow2^{\allPerformances}$ that gives the set of performances that are achievable for sure, for some set of achievable performances given in input. Note that $\achievableByCombinations$ is idempotent, \ie $\achievableByCombinations\circ\achievableByCombinations=\achievableByCombinations$.

Consider a random experiment using a black box entity $\anEntity$ only once (this is the knowledge we have about $\evaluation$). Let $\lambda_1, \lambda_2, \ldots \lambda_n$ be positive values summing up to one. All other things remaining identical, if one can achieve the performances $\aPerformance_1, \aPerformance_2, \ldots \aPerformance_n$ with, respectively, the entities $\anEntity_1, \anEntity_2, \ldots \anEntity_n$, then the performance $\sum_i\lambda_i \aPerformance_i$ is achievable with a hybrid entity that randomly selects an entity among $\anEntity_1, \anEntity_2, \ldots \anEntity_n$ with the series of respective selection probabilities $\lambda_1, \lambda_2, \ldots \lambda_n$ before running the corresponding entity. In this case, denoting the set of all possible convex combinations by $\allConvexCombinations$, we can take $\achievableByCombinations=\allConvexCombinations$.

\addtocounter{example}{-1}
\begin{example}[continued]
    In the particular case of two-class crisp classification, when the source of samples, the oracle, and the descriptor are kept unchanged, \emph{Fawcett's interpolation}~\cite{Fawcett2006AnIntroduction} allows interpolating linearly between performances with a hybrid classifier. And, when the oracle, the descriptor, and the classifier are kept unchanged, \emph{Pi\'erard's summarization}~\cite{Pierard2020Summarizing} allows interpolating linearly between performances with a hybrid source of samples.
\end{example}

\subsection{The Scores as Functions $\aScore$ of Performances}
\label{sec:mathematical-framework-scores}

\newcommandx\expectedValueScore[1]{\aScore_{#1}^{E}}%
\newcommandx\probabilisticScore[1]{\aScore_{#1}^{P}}%

In our framework, \emph{scores}\footnote{We choose the term \emph{score} to avoid any possible confusion with the mathematical meaning of the terms \emph{metric}, \emph{measure}, and \emph{indicator}.} (also called metrics, measures, indicators, criteria, factors, and indices in the literature~\cite{Texel2013Measure,Canbek2017Binary}) are functions associating a real value to performances, that is, $\aScore:\domainOfScore\rightarrow\realNumbers:\aPerformance\mapsto\aScore(\aPerformance)$ 
with $\domainOfScore\subseteq\allPerformances$. 
One can define different parametric families of scores, such as: 

\begin{itemize}
    \item \emph{Expected value scores} are parameterized by a random variable $V$. We define them as $\expectedValueScore{V}:\allPerformances\rightarrow[\min_{\aSample\in\sampleSpace}V(\aSample),\max_{\aSample\in\sampleSpace}V(\aSample)]:\aPerformance\mapsto\expectedValueScore{V}(\aPerformance)=\expectedValueSymbol_\aPerformance[V]$, where $\expectedValueSymbol$ denotes the mathematical expectation. The score $\scoreExpectedSatisfaction$, that we call the \emph{expected satisfaction}, is a universal score in the sense that it exists for all sample spaces $\sampleSpace$. It will be further studied in \cref{sec:ranking-theory}.
    \item \emph{Probabilistic scores} are parameterized by two events $\anEvent_1,\anEvent_2\in\eventSpace$ such that $\emptyset\subsetneq \anEvent_1\subsetneq \anEvent_2\subseteq\sampleSpace$. We define them as $\probabilisticScore{\anEvent_1\vert \anEvent_2}:\{\aPerformance\in\allPerformances:\aPerformance(\anEvent_2)\ne0\}\rightarrow[0,1]:\aPerformance\mapsto\probabilisticScore{\anEvent_1\vert \anEvent_2}(\aPerformance)=\aPerformance(\anEvent_1\vert \anEvent_2)$. All probabilistic scores can be expressed as a ratio of two expected value scores.
\end{itemize}

\addtocounter{example}{-1}
\begin{example}[continued]
    In the case of two-class classification, the expected satisfaction is both the expected value score $\expectedValueScore{\randVarSatisfaction}$ and the probabilistic score $\probabilisticScore{\randVarSatisfaction=1\vert \sampleSpace}$. This is because $\randVarSatisfaction$ is a $\{0,1\}$-binary random variable. In this case, the expected satisfaction is called \emph{accuracy} and is noted $\scoreAccuracy$.
\end{example}

\subsection{Modeling Application-Specific Preferences as a Random Variable $\randVarImportance$}
\label{sec:mathematical-framework-importance}

It is well known that the ranking of entities does not only have to be specific for the task, but that it should also be sensitive to application-specific preferences. To encode these preferences, we propose to rely on a second random variable that we call \emph{importance}: $\randVarImportance:\sampleSpace\rightarrow\realNumbers_{\ge 0}$. We require that $\randVarImportance\ne0$, \ie $\exists\aSample:\randVarImportance(\aSample)\ne0$.

Further in this paper, we will describe a new family of scores, called \emph{ranking scores}, which are parameterized by this random variable. We would, however, like to draw the reader's attention to the fact that the importance is something that cannot, in general, be deduced from a score. In particular, the visual inspection of a formula for a given score could be misleading. Therefore, in this paper, we present a technique to analyze the behavior of any score by computing the rank correlations with the ranking scores for which the importances are well-defined.

\addtocounter{example}{-1}
\begin{example}[continued]
    In the case of two-class classification, we provide examples of misleading formulas in~\cref{sup:misleading_visual_inspection}, in particular two equivalent formulas for the accuracy and two for the true positive rate. In both cases, by visually inspecting them, one would intuitively draw different conclusions about the importance given to the true negatives, false positives, false negatives, and true positives.
\end{example}

\section{Performance-Based Ranking Theory}\label{sec:ranking-theory}

In this section, we build our theory in three steps, represented by the three lintels of~\cref{fig:graphical_abstract}, on top of the six pillars depicting our mathematical framework. 
First, we present a universal axiomatic definition of performance orderings and performance-based rankings of entities (\cref{sec:axioms}). Then, we link these axioms with the scores, and give a sufficient condition per axiom (\cref{sec:sufficient-conditions}). Finally, we account for the application-specific preferences and provide an infinite, diversified, and universal family of scores that induce performance orderings satisfying our axioms (\cref{sec:ranking-scores}).

\subsection{Axiomatic Definition}\label{sec:axioms}

We propose an axiomatic definition (the first as far as we know) of both performance orderings and performance-based rankings. 
The axioms do not involve any $\aScore$ or $\randVarImportance$.

\subsubsection{Leveraging the Preorder $\relWorseOrEquivalent$}

We argue that if several entities from a set $\entitiesToRank$ have been ranked, then removing or adding an entity should not affect the relative order of those ranked entities. To guarantee it, our first axiom imposes that the ranking function is based on a preorder $\relWorseOrEquivalent$ on $\allPerformances$. There is no consensus in the literature on the rank value to consider when several entities have equivalent performances. Thus, instead of setting an arbitrarily chosen value, our axiom specifies bounds for it.
\begin{restatable}{axiom}{restatableAA}
\label{axiom:preorder}
The ranking function $\rank:\entitiesToRank\rightarrow[1,|\entitiesToRank|]:\anEntity\mapsto\rank(\anEntity)$ satisfies $|\{\anEntity'\in\entitiesToRank:\evaluation(\anEntity)\relWorse\evaluation(\anEntity')\}|+1\le\rank(\anEntity)\le|\{\anEntity'\in\entitiesToRank:\evaluation(\anEntity)\relWorseOrEquivalent\evaluation(\anEntity')\}|$, where $\relWorseOrEquivalent$ is  a preorder on $\allPerformances$. 
\end{restatable}

\subsubsection{Leveraging the Satisfaction $\randVarSatisfaction$}

We argue that if the satisfaction that can be obtained with an entity $\anEntity_1$ is for sure less or equal than the satisfaction that can be obtained with an entity $\anEntity_2$, then the performance $\aPerformance_1$ of $\anEntity_1$ cannot be better than the performance $\aPerformance_2$ of $\anEntity_2$.

\begin{restatable}{axiom}{restatableAB}
\label{axiom:satisfaction}
For $\aPerformance_1,\aPerformance_2\in\allPerformances$ such that $\aPerformance_1(\randVarSatisfaction\le s)=1$ and $\aPerformance_2(\randVarSatisfaction\ge s)=1$ for some $s$, then $\aPerformance_1\relWorseOrEquivalent\aPerformance_2$ or $\aPerformance_1\relIncomparable\aPerformance_2$.
\end{restatable}

This implies that, thanks to $\randVarSatisfaction$, we can use the natural order $\le$ that exists on $\realNumbers$ to obtain a preorder on $\sampleSpace$, and the axiom states that the preorder on $\allPerformances$ is coherent with it.

Let us now take a look at three implications that are clearly intuitively expected.

\begin{corollary}
    \label{corollary:quivalent-incomparable-performances}
    For any $s$, all performances $\aPerformance$ such that $\aPerformance(\randVarSatisfaction=s)=1$ are either equivalent or incomparable.
\end{corollary}

\begin{corollary}
    \label{corollary:worst-performances}
    We have $\aPerformance(\randVarSatisfaction=\minSatisfaction)=1 \Rightarrow \nexists\aPerformance':\aPerformance'\relWorse\aPerformance$. In other words, $\aPerformance(\randVarSatisfaction=\minSatisfaction)=1$ means that $\aPerformance$ belongs to the set of the worst performances, among $\allPerformances$.
\end{corollary}

\begin{corollary}
    \label{corollary:best-performances}
    We have $\aPerformance(\randVarSatisfaction=\maxSatisfaction)=1 \Rightarrow \nexists\aPerformance':\aPerformance'\relBetter\aPerformance$. In other words, $\aPerformance(\randVarSatisfaction=\maxSatisfaction)=1$ means that $\aPerformance$ belongs to the set of the best performances, among $\allPerformances$.
\end{corollary}

\subsubsection{Leveraging the Combinations $\achievableByCombinations$}

We argue that, for any set of achievable performances, it must be impossible to obtain with certainty a performance better than the best of them, or worse than the worst of them, by sequentially combining operations from an arbitrary set of possible operations to perturb (\eg, add noise to their output) the entities corresponding to that initial set. Expressing this requirement in terms of $\relWorseOrEquivalent$ leads to our third axiom.

\begin{restatable}{axiom}{restatableAD}
\label{axiom:combinations}
Let $\aPerformance$ be a performance, and $\aSetOfPerformances$ be a set of performances on $\allPerformances$ such that $\aPerformance'\relWorseOrEquivalent\aPerformance\vee\aPerformance\relWorseOrEquivalent\aPerformance'\,\forall\aPerformance'\in\aSetOfPerformances$.
\begin{itemize}
    \item $\aPerformance'\relWorseOrEquivalent\aPerformance\, \forall\aPerformance'\in\aSetOfPerformances \Rightarrow \overline{\aPerformance}\relWorseOrEquivalent\aPerformance\, \forall\overline{\aPerformance}\in\achievableByCombinations(\aSetOfPerformances)$;
    \item $\aPerformance'\not\relWorseOrEquivalent\aPerformance\, \forall\aPerformance'\in\aSetOfPerformances \Rightarrow \overline{\aPerformance}\not\relWorseOrEquivalent\aPerformance\, \forall\overline{\aPerformance}\in\achievableByCombinations(\aSetOfPerformances)$;
    \item $\aPerformance\relWorseOrEquivalent\aPerformance'\, \forall\aPerformance'\in\aSetOfPerformances \Rightarrow \aPerformance\relWorseOrEquivalent\overline{\aPerformance}\, \forall\overline{\aPerformance}\in\achievableByCombinations(\aSetOfPerformances)$;
    \item and $\aPerformance\not\relWorseOrEquivalent\aPerformance'\, \forall\aPerformance'\in\aSetOfPerformances \Rightarrow \aPerformance\not\relWorseOrEquivalent\overline{\aPerformance}\, \forall\overline{\aPerformance}\in\achievableByCombinations(\aSetOfPerformances)$.
\end{itemize}
\end{restatable}

With the definitions for $\relEquivalent$, $\relBetter$, $\relWorse$, and $\relIncomparable$ given in \cref{sec:mathematical-framework-ordering}, the following corollary can be derived.

\begin{corollary}
Let $\aPerformance\in\allPerformances$ and $\aSetOfPerformances\subseteq\allPerformances$ such that $\aPerformance'\relWorseOrEquivalent\aPerformance\vee\aPerformance\relWorseOrEquivalent\aPerformance'\,\forall\aPerformance'\in\aSetOfPerformances$. We have:
\begin{itemize}
    \item $\aPerformance'\relEquivalent\aPerformance\, \forall\aPerformance'\in\aSetOfPerformances \Rightarrow \overline{\aPerformance}\relEquivalent\aPerformance\, \forall\overline{\aPerformance}\in\achievableByCombinations(\aSetOfPerformances)$;
    \item $\aPerformance'\relBetter\aPerformance\, \forall\aPerformance'\in\aSetOfPerformances \Rightarrow \overline{\aPerformance}\relBetter\aPerformance\, \forall\overline{\aPerformance}\in\achievableByCombinations(\aSetOfPerformances)$;
    \item $\aPerformance'\relWorse\aPerformance\, \forall\aPerformance'\in\aSetOfPerformances \Rightarrow \overline{\aPerformance}\relWorse\aPerformance\, \forall\overline{\aPerformance}\in\achievableByCombinations(\aSetOfPerformances)$;
    \item $\aPerformance'\relIncomparable\aPerformance\, \forall\aPerformance'\in\aSetOfPerformances \Rightarrow \overline{\aPerformance}\relIncomparable\aPerformance\, \forall\overline{\aPerformance}\in\achievableByCombinations(\aSetOfPerformances)$.
\end{itemize}
\end{corollary}

\subsubsection{Consistency}


The three axioms are consistent in the sense that they do not contradict each others. A trivial preorder $\relWorseOrEquivalent$ for which all axioms are satisfied is the one such that all performances are equivalent, regardless of what $\randVarSatisfaction$ and $\achievableByCombinations$ are.

\subsection{Sufficient Conditions for Score-Based Rankings}
\label{sec:sufficient-conditions}

We can now connect the axioms and scores and give sufficient conditions to satisfy our axioms. The proofs for the three following theorems are given in~\cref{sec:sufficient-conditions-proofs}.

The \first theorem explains how performance orderings $\ordering$ can be induced from scores $\aScore$, which allows capitalizing on the natural order $\le$  on $\realNumbers$ to obtain a preorder $\relWorseOrEquivalent$ on $\allPerformances$.

\begin{restatable}[Sufficient condition for Axiom~\ref{axiom:preorder}]{theorem}{restatableTA}\label{thm:suff-cond-preorder} 
    A binary relation $\ordering_\aScore$ on $\allPerformances$ induced by a score $\aScore$ as $\aPerformance_1 \ordering_\aScore \aPerformance_2$ \iif either $\aPerformance_1=\aPerformance_2$ or $\aPerformance_1\in\domainOfScore$ and $\aPerformance_2\in\domainOfScore$ and $\aScore(\aPerformance_1)\le\aScore(\aPerformance_2)$, is a preorder satisfying Axiom~\ref{axiom:preorder}.
\end{restatable}

The \second theorem makes the connection between the properties of the scores $\aScore$ and the satisfaction $\randVarSatisfaction$.

\begin{restatable}[Sufficient condition for Axiom~\ref{axiom:satisfaction}]{theorem}{restatableTBC} \label{thm:suff-cond-satisfaction} 
    If a score $\aScore$ satisfies $\min_{\aSample\in\anEvent}\randVarSatisfaction(\aSample) \le \aScore(\aPerformance) \le \max_{\aSample\in\anEvent}\randVarSatisfaction(\aSample)$  for all events $\anEvent\in\eventSpace$ and all performances $\aPerformance\in\domainOfScore[\aScore]$ such that $\aPerformance(\anEvent)=1$, then the ordering $\ordering_{\aScore}$ satisfies Axiom~\ref{axiom:satisfaction}.
\end{restatable}

The \third theorem makes the connection between the properties of the scores $\aScore$ and the function $\achievableByCombinations$.

\begin{restatable}[Sufficient condition for Axiom~\ref{axiom:combinations}]{theorem}{restatableTD}\label{thm:suff-cond-combinations}
    If a score $\aScore$ is such that $\aSetOfPerformances\subseteq\domainOfScore\Rightarrow\achievableByCombinations(\aSetOfPerformances)\subseteq\domainOfScore$ and  $\min_{\aPerformance\in\aSetOfPerformances}\aScore(\aPerformance) \le \aScore(\overline{\aPerformance}) \le \max_{\aPerformance\in\aSetOfPerformances} \aScore(\aPerformance)$ for all $\aSetOfPerformances\subseteq\domainOfScore$ and all $\overline{\aPerformance}\in\achievableByCombinations(\aSetOfPerformances)$, then the ordering $\ordering_{\aScore}$ satisfies Axiom~\ref{axiom:combinations}.
\end{restatable}

\subsection{The Ranking Scores: Solutions for \texorpdfstring{$\achievableByCombinations=\allConvexCombinations$}{phi=conv}}
\label{sec:ranking-scores}

We now aim to define scores that satisfy our theorems and thus our axioms. We provide such scores in the case of $\achievableByCombinations=\allConvexCombinations$, shown particularly relevant in our catalog of practical examples (see~\cref{sec:catalog-of-problems}). Beyond that, we include the application-specific preferences modeled by the random variable $\randVarImportance$. We introduce a new family of scores, the \emph{ranking scores} $\rankingScore$, that are parameterized by the \emph{importance}, a non-negative random variable $\randVarImportance\ne0$:
\begin{multline*}
    \rankingScore:\domainOfScore[\rankingScore]\rightarrow[\minSatisfaction,\maxSatisfaction]:\\
    \aPerformance\mapsto
    \rankingScore(\aPerformance)=\frac{
        \expectedValueSymbol_{\aPerformance}[\randVarImportance\randVarSatisfaction]
    }{
        \expectedValueSymbol_{\aPerformance}[\randVarImportance]
    }
    =\frac{
        \sum_{\aSample\in\sampleSpace}\randVarImportance(\aSample)\randVarSatisfaction(\aSample)\aPerformance(\{\aSample\})
    }{
        \sum_{\aSample\in\sampleSpace}\randVarImportance(\aSample)\aPerformance(\{\aSample\})
    }
\end{multline*}
where $\domainOfScore[\rankingScore]=\{\aPerformance\in\allPerformances:\expectedValueSymbol_{\aPerformance}[\randVarImportance]\ne0\}$ and the second equality holds in the common case where $\sampleSpace$ is finite and $\eventSpace=2^\sampleSpace$. The expected satisfaction $\scoreExpectedSatisfaction$ corresponds to $\rankingScore$ when all samples are equally important.

The scores $\rankingScore$ and the performance orderings $\ordering_{\rankingScore}$ satisfy the conditions of Theorem~\ref{thm:suff-cond-preorder}, \ref{thm:suff-cond-satisfaction}, and~\ref{thm:suff-cond-combinations} for $\achievableByCombinations=\allConvexCombinations$ (proofs in~\cref{sec:ranking-scores-proofs}). Thus, the performance orderings $\ordering_{\rankingScore}$ induced by the ranking scores satisfy all our axioms. Hence, the chosen name. Remarkably, these scores are universal: they can be used for performance-based ranking, regardless of the sample space $\sampleSpace$.

\paragraph{Properties.}
Let us now give some selected properties (proofs are in~\cref{sec:ranking-scores-proofs}).  The first one clarifies how the importance can be interpreted. Consider any random experiment for which the set of possible outcomes is $\sampleSpace$, their distribution being given by $\aPerformance$. One can choose to \emph{filter} the outcomes as follows: run the experiment, and if the outcome is $\aSample$, end the experiment with probability $\randVarImportance(\aSample)/\sum_{\aSample'\in\sampleSpace}\randVarImportance(\aSample')$, otherwise restart everything. The new distribution of outcomes is given by the operation $\opFilter(\aPerformance)$.

\begin{property}
\label[property]{prop:decomposition-importance-satisfaction}
    The ranking scores can be decomposed into an operation on performances considering the importance $\randVarImportance$ and a score (the expected satisfaction) considering the satisfaction $\randVarSatisfaction$. We have $\rankingScore=\scoreExpectedSatisfaction\circ\opFilter$, with 
    \begin{multline*}
        \opFilter:\domainOfScore[\rankingScore]\rightarrow\allPerformances:\\
        \aPerformance\mapsto\left(\eventSpace\rightarrow[0,1]:\anEvent\mapsto\frac{\sum_{\aSample\in\anEvent}\aPerformance(\{\aSample\})\randVarImportance(\aSample)}{\sum_{\aSample\in\sampleSpace}\aPerformance(\{\aSample\})\randVarImportance(\aSample)}\right).
    \end{multline*}
\end{property}

So far, the satisfaction was fixed. But what if one hesitates about the objective of the task, \ie with its modeling through $\randVarSatisfaction$? The next property clarifies what really matters.

\begin{property}
\label[property]{prop:linear-transformation-satisfaction}
    Linearly transforming the satisfaction results in the same linear transformation of the ranking score. It is something that does not affect the ordering.
\end{property}

The next property is about the scale invariance of $\randVarImportance$.

\begin{property}
\label[property]{prop:scale-invariance}
    $\rankingScore[k\randVarImportance]=\rankingScore[\randVarImportance], \, \forall k\ne0$. We can thus restrict the study of ranking scores to the case in which the total importance $\sum_{\aSample\in\eventSpace}\randVarImportance(\aSample)$ is constant.
\end{property}

We can go further in the case of a binary satisfaction.

\begin{property}
\label[property]{prop:scale-invariance-per-satisfaction}
    For a binary satisfaction, the performance ordering induced by a ranking score is insensitive to the uniform scaling of the importance given to the unsatisfying ($\randVarSatisfaction^{-1}(0)$) or to the satisfying ($\randVarSatisfaction^{-1}(1)$) samples.
\end{property}

Most scores obtained by averaging or integrating ranking scores are not themselves ranking scores, and are unsuited for ranking. Thanks to Theorem~\ref{thm:suff-cond-satisfaction}, we know that the performance orderings induced by them satisfy Axiom~\ref{axiom:satisfaction}. However, most often, they do not satisfy Axiom~\ref{axiom:combinations}. Yet, two exceptions occur in the case of a binary satisfaction. 

\begin{property}
\label[property]{prop:mean-vertical}
    If $\randVarImportance$ is the arithmetic mean of $\randVarImportance_1$ and $\randVarImportance_2$, then $\rankingScore[\randVarImportance]$ is the $f$-mean of $\rankingScore[\randVarImportance_1]$ and $\rankingScore[\randVarImportance_2]$ with $f:x\mapsto x^{-1}$, \ie, the harmonic mean, when $\randVarSatisfaction(\aSample)=1\Rightarrow\randVarImportance_1(\aSample)=\randVarImportance_2(\aSample)$.
\end{property}

\begin{property}
\label[property]{prop:mean-horizontal}
    If $\randVarImportance$ is the arithmetic mean of $\randVarImportance_1$ and $\randVarImportance_2$, then $\rankingScore[\randVarImportance]$ is the $f$-mean of $\rankingScore[\randVarImportance_1]$ and $\rankingScore[\randVarImportance_2]$ with $f:x\mapsto (1-x)^{-1}$, when $\randVarSatisfaction(\aSample)=0\Rightarrow\randVarImportance_1(\aSample)=\randVarImportance_2(\aSample)$.
\end{property}

For any homogeneous binary relation $\aRelation$ (\eg $\le$, $<$, $=$, $>$, $\ge$, \ldots) on $\realNumbers$, let us define the set $\phi_\aRelation(\aPerformance)=\{\aPerformance'\in\allPerformances:\rankingScore(\aPerformance')\aRelation\rankingScore(\aPerformance)\}$. Based on these sets, the following property helps to understand why the ranking scores are suitable when $\achievableByCombinations=\allConvexCombinations$. This is indeed related to the fact that the scores $\rankingScore$ are pseudolinear functions. 

\begin{property}
\label[property]{prop:convexity-contour-sets}
    With the performance orderings $\ordering_{\rankingScore}$ induced by the ranking scores $\rankingScore$ and for any given performance $\aPerformance\in\allPerformances$, the worse set $\phi_<(\aPerformance)$ (\aka the strictly lower contour set), the worse or equivalent set $\phi_\le(\aPerformance)$ (the lower contour set), the equivalent set $\phi_=(\aPerformance)$, the better or equivalent set $\phi_\ge(\aPerformance)$ (the upper contour set) and the better set $\phi_>(\aPerformance)$ (the strict upper contour set) are all convex.
\end{property}

\section{Observations for Two-Class Classification}
\label{sec:two-class-crisp-classification}

We now examine the particular case of two-class crisp classification. First, we compare the classical formulation of this task with ours. Then, we examine more closely what our ranking theory teaches us for this task.

\subsection{Classical vs. our Formulation}

\mysection{Limitation of the classical formulation.} Usual two-class classification settings first define a set $\allClasses=\{\classNeg,\classPos\}$ in which $\classNeg$ and $\classPos$ are named \emph{negative class} and \emph{positive class}. Then, pairs $(y,\hat{y})\in \allClasses^2$ composed of a ``ground truth'' $y$ and a ``prediction'' $\hat{y}$ are interpreted as: a \emph{true negative} $\sampleTN=(\classNeg,\classNeg)$, a \emph{false positive} $\sampleFP=(\classNeg,\classPos)$ (type I error), a \emph{false negative} $\sampleFN=(\classPos,\classNeg)$ (type II error), and a \emph{true positive} $\sampleTP=(\classPos,\classPos)$. Finally, scores are defined as formulas involving these elements, upon which rankings of entities are based. However, there is no mathematical guarantee of the meaningfulness of such rankings.

\mysection{Advantage of our formulation.} In our framework, the two-class classification settings are as follows. We consider the sample space $\sampleSpace=\{\sampleTN,\sampleFP,\sampleFN,\sampleTP\}$ and the event space $\eventSpace=2^\sampleSpace$. The samples $\sampleTN$, $\sampleFP$, $\sampleFN$, and $\sampleTP$ are interpreted as in the classical case. The natural choice for the satisfaction is such that $\randVarSatisfaction(\sampleFP)=\randVarSatisfaction(\sampleFN)=0$ and $\randVarSatisfaction(\sampleTN)=\randVarSatisfaction(\sampleTP)=1$. 
After modeling application-specific preferences through a choice of importances $\randVarImportance$, we can derive rankings scores $\rankingScore[\randVarImportance]$ to rank entities based on their performances, with the certainty that the ranking is mathematically valid.

\mysection{Link between the two formulations.} Fortunately, the two formulations can be easily connected, which allows translating concepts from one to the other whenever necessary. Indeed, to go from our formulation to the classical one, we only need to define a ``ground-truth'' random variable $\randVarGroundtruthClass:\sampleSpace\rightarrow \allClasses$ such that $\randVarGroundtruthClass(\sampleTN)=\randVarGroundtruthClass(\sampleFP)=\classNeg$ and $\randVarGroundtruthClass(\sampleFN)=\randVarGroundtruthClass(\sampleTP)=\classPos$, as well as the ``prediction'' random variable $\randVarPredictedClass:\sampleSpace\rightarrow \allClasses$ such that $\randVarPredictedClass(\sampleTN)=\randVarPredictedClass(\sampleFN)=\classNeg$ and $\randVarPredictedClass(\sampleFP)=\randVarPredictedClass(\sampleTP)=\classPos$. Conversely, moving from the classical formulation to ours just requires considering $\sampleSpace=\allClasses^2$ and 
$\randVarSatisfaction=\indicatorSymbol_{\randVarGroundtruthClass=\randVarPredictedClass}$. This link is represented in~\cref{fig:formulations}.

\subsection{What Does our Ranking Theory Teach us?}
\label{sec:review}

\begin{table*}
\newcommandx\oka[1][addprefix=\global]{\textcolor{ForestGreen}{\bf #1}}%
\newcommandx\okax[1][addprefix=\global]{\oka{1\textsuperscript{\textdagger}}}%
\newcommandx\okb[1][addprefix=\global]{\textcolor{Orange}{#1}}%
\newcommandx\okbx[1][addprefix=\global]{\okb{1\textsuperscript{\textdagger}}}%
\newcommandx\okc[1][addprefix=\global]{\textcolor{Gray}{#1}}%
\newcommandx\okcx[1][addprefix=\global]{\okc{0\textsuperscript{\textdagger}}}%
\newcommandx\kodx[1][addprefix=\global]{-1\textsuperscript{\textdagger}}%

\caption{Properties of some common scores defined in the literature for two-class crisp classification. The symbol \textdagger indicates a value that has been obtained theoretically, the others have been obtained empirically. Our conclusion is that, for the purpose of ranking, the scores in green can always be used, those in orange should only be used when the priors are fixed, and those in black cannot be used even when the priors are fixed. See \cref{sec:two-class-crisp-classification} for the detailed description.\label{tbl:scores-review}}

\resizebox{1\linewidth}{!}{

\begin{tabular}{|l|c|c|c|c|c|c|c|c|c|c|c|c|c|c|c|}
\cline{2-16} \cline{3-16} \cline{4-16} \cline{5-16} \cline{6-16} \cline{7-16} \cline{8-16} \cline{9-16} \cline{10-16} \cline{11-16} \cline{12-16} \cline{13-16} \cline{14-16} \cline{15-16} \cline{16-16} 
\multicolumn{1}{c|}{} & \multicolumn{5}{c|}{without any constraint: all performances} & \multicolumn{5}{c|}{with constraint: positive prior = $0.2$} & \multicolumn{5}{c|}{with constraint: positive prior = $0.5$}\tabularnewline
\hline 
score & 
\first test & \second test & \third test & $\tau_{min}$ & $\tau_{max}$ & 
\first test & \second test & \third test & $\tau_{min}$ & $\tau_{max}$ & 
\first test & \second test & \third test & $\tau_{min}$ & $\tau_{max}$ \tabularnewline
\hline 
\oka{Accuracy}                                                          & \oka{V} & \oka{V} & \oka{V} & \oka{0.469} & \okax{0.982} & \oka{V} & \oka{V} & \oka{V} & \oka{0.157} & \okax{0.998} & \oka{V} & \oka{V} & \oka{V} & \oka{0.505} & \okax{0.995} \tabularnewline
\oka{F-score for $\beta=0.5$}                                              & \oka{V} & \oka{V} & \oka{V} & \oka{0.079} & \okax{1.000} & \oka{V} & \oka{V} & \oka{V} & \oka{0.451} & \okax{1.000} & \oka{V} & \oka{V} & \oka{V} & \oka{0.352} & \okax{1.000} \tabularnewline
\oka{F-score for $\beta=1.0$}                                              & \oka{V} & \oka{V} & \oka{V} & \oka{0.161} & \okax{0.994} & \oka{V} & \oka{V} & \oka{V} & \oka{0.352} & \okax{1.000} & \oka{V} & \oka{V} & \oka{V} & \oka{0.194} & \okax{1.000} \tabularnewline
\oka{F-score for $\beta=2.0$}                                              & \oka{V} & \oka{V} & \oka{V} & \oka{0.079} & \okax{1.000} & \oka{V} & \oka{V} & \oka{V} & \oka{0.194} & \okax{1.000} & \oka{V} & \oka{V} & \oka{V} & \oka{0.072} & \okax{1.000} \tabularnewline
\oka{Negative Predictive Value (NPV)}                                   & \oka{V} & \oka{V} & \oka{V} & \oka{0.000} & \okax{1.000} & \oka{V} & \oka{V} & \oka{V} & \oka{0.503} & \okax{1.000} & \oka{V} & \oka{V} & \oka{V} & \oka{0.503} & \okax{1.000} \tabularnewline
\oka{Positive Predictive Value (PPV)}                                   & \oka{V} & \oka{V} & \oka{V} & \oka{0.000} & \okax{1.000} & \oka{V} & \oka{V} & \oka{V} & \oka{0.503} & \okax{1.000} & \oka{V} & \oka{V} & \oka{V} & \oka{0.503} & \okax{1.000} \tabularnewline
\oka{True Negative Rate (TNR)}                                          & \oka{V} & \oka{V} & \oka{V} & \oka{0.000} & \okax{1.000} & \oka{V} & \oka{V} & \oka{V} & \oka{0.000} & \okax{1.000} & \oka{V} & \oka{V} & \oka{V} & \oka{0.000} & \okax{1.000} \tabularnewline
\oka{True Positive Rate (TPR)}                                          & \oka{V} & \oka{V} & \oka{V} & \oka{0.000} & \okax{1.000} & \oka{V} & \oka{V} & \oka{V} & \oka{0.000} & \okax{1.000} & \oka{V} & \oka{V} & \oka{V} & \oka{0.000} & \okax{1.000} \tabularnewline
\hline
\okb{Balanced Accuracy}                                                 & V & X & X & 0.486 & 0.713 & \okb{V} & \okb{V} & \okb{V} & \okb{0.504} & \okbx{0.997} & \okb{V} & \okb{V} & \okb{V} & \okb{0.505} & \okbx{0.995} \tabularnewline
\okb{Cohen's $\kappa$}                                                     & X & X & X & 0.476 & 0.697 & \okb{V} & \okb{V} & \okb{V} & \okb{0.503} & \okbx{1.000} & \okb{V} & \okb{V} & \okb{V} & \okb{0.505} & \okbx{0.995} \tabularnewline
\okb{Informedness}                                                      & V & X & X & 0.486 & 0.713 & \okb{V} & \okb{V} & \okb{V} & \okb{0.504} & \okbx{0.997} & \okb{V} & \okb{V} & \okb{V} & \okb{0.505} & \okbx{0.995} \tabularnewline
\okb{Positive Likelihood Ratio (PLR)}                                   & V & X & X & 0.420 & 0.677 & \okb{V} & \okb{V} & \okb{V} & \okb{0.491} & \okbx{1.000} & \okb{V} & \okb{V} & \okb{V} & \okb{0.491} & \okbx{1.000} \tabularnewline
\okb{Probability of True Negative (PTN)}                                & X & V & V & -0.007 & 0.818 & \okb{V} & \okb{V} & \okb{V} & \okb{0.000} & \okbx{1.000} & \okb{V} & \okb{V} & \okb{V} & \okb{0.000} & \okbx{1.000} \tabularnewline
\okb{Probability of True Positive (PTP)}                                & X & V & V & -0.006 & 0.818 & \okb{V} & \okb{V} & \okb{V} & \okb{0.000} & \okbx{1.000} & \okb{V} & \okb{V} & \okb{V} & \okb{0.000} & \okbx{1.000} \tabularnewline
\hline
Chance in Cohen's $\kappa$                                           & X & X & X & 0.194 & 0.498 & X & V & V & -0.157 & 0.849 & \okc{V} & \okc{V} & \okc{V} & \okcx{-0.012} & \okcx{0.008} \tabularnewline
Error Rate                                                        & X & V & V & \kodx{-0.982} & -0.469 & X & V & V & \kodx{-0.998} & -0.157 & X & V & V & \kodx{-0.995} & -0.505 \tabularnewline
False Discovery Rate (FDR)                                        & X & V & V & \kodx{-1.000} & 0.000 & X & V & V & \kodx{-1.000} & -0.503 & X & V & V & \kodx{-1.000} & -0.503 \tabularnewline
False Negative Rate (FNR)                                         & X & V & V & \kodx{-1.000} & 0.000 & X & V & V & \kodx{-1.000} & 0.000 & X & V & V & \kodx{-1.000} & 0.000 \tabularnewline
False Omission Rate (FOR)                                         & X & V & V & \kodx{-1.000} & 0.000 & X & V & V & \kodx{-1.000} & -0.503 & X & V & V & \kodx{-1.000} & -0.503 \tabularnewline
False Positive Rate (FPR)                                         & X & V & V & \kodx{-1.000} & 0.000 & X & V & V & \kodx{-1.000} & 0.000 & X & V & V & \kodx{-1.000} & 0.000 \tabularnewline
Geometric mean of TNR and TPR                                     & V & X & X & 0.461 & 0.653 & V & X & V & 0.503 & 0.831 & V & X & V & 0.503 & 0.830 \tabularnewline
Markedness                                                        & V & X & X & 0.486 & 0.713 & V & X & X & 0.418 & 0.887 & V & X & X & 0.503 & 0.913 \tabularnewline
Matthews Correlation Coefficient (MCC)                            & V & X & X & 0.503 & 0.746 & V & X & X & 0.458 & 0.944 & V & X & X & 0.503 & 0.963 \tabularnewline
Negative Likelihood Ratio (NLR)                                   & X & X & X & -0.677 & -0.418 & X & V & V & \kodx{-1.000} & -0.491 & X & V & V & \kodx{-1.000} & -0.491 \tabularnewline
Odds Ratio (OR)                                                   & V & X & X & 0.499 & 0.671 & V & X & X & 0.503 & 0.894 & V & X & X & 0.503 & 0.892 \tabularnewline
Rate of positive predictions                                      & X & V & V & -0.469 & 0.469 & X & V & V & -0.849 & 0.157 & X & V & V & -0.504 & 0.505 \tabularnewline
Sensitivity Index Estimate ($d'$)                                   & V & X & X & 0.502 & 0.786 & V & X & X & 0.503 & 0.926 & V & X & X & 0.503 & 0.924 \tabularnewline
\hline 
\end{tabular}
\par
}
\end{table*}

We now examine some common scores used in the literature for two-class classification, from the ranking standpoint. Methodologically, we choose the set of numerical quantities that were listed in a recent review~\cite{Canbek2017Binary}, and focus on those that are scores (this excludes the base measures and the \first level measures as they are called in that paper).

We consider three sets of performances: (1) all performances $\allPerformances$; (2) all performances for fixed and unbalanced priors (we set arbitrarily a positive prior of $0.2$); (3) all performances for fixed and balanced priors.
For each score $\aScore$ and each set of performances $\aSetOfPerformances$, we report in \cref{tbl:scores-review} the result of three tests on $\aScore$ as well as the minimum and maximum rank correlations $\aScore$ has with the ranking scores.

The tests are the following. The \first test determines if $\ordering_{\aScore}$ satisfies Axiom~\ref{axiom:satisfaction} when the set $\allPerformances$ is restricted to $\aSetOfPerformances$. The \second and \third tests relate to \cref{thm:suff-cond-combinations} (and thus to Axiom~\ref{axiom:combinations}). The \second test determines if, for any subset $\aSetOfPerformances'$ of $\aSetOfPerformances$, $\max_{\aPerformance'\in\aSetOfPerformances'}\aScore(\aPerformance')$ is greater or equal to $\aScore(\aPerformance)$ for all performances $\aPerformance\in\allConvexCombinations(\aSetOfPerformances')$. The \third test determines if, for any subset $\aSetOfPerformances'$ of $\aSetOfPerformances$, $\min_{\aPerformance'\in\aSetOfPerformances'}\aScore(\aPerformance')$ is less or equal to $\aScore(\aPerformance)$ for all performances $\aPerformance\in\allConvexCombinations(\aSetOfPerformances')$.

For the rank correlations, we first try to determine analytically if the score is monotonically increasing or decreasing with one of the ranking scores. If it is the case, we report a maximum correlation of $1$ or a minimum correlation of $-1$, respectively. The proofs can be found in~\cref{sec:perfect-correlation-proofs-1,sec:perfect-correlation-proofs-2}. Otherwise, we report empirical values obtained by optimizing Kendall's $\tau$. We consider a uniform distribution of performances within $\aSetOfPerformances$. Kendall's $\tau$ is computed using a function provided in \textsc{SciPy}~\cite{Virtanen2020SciPy}, with its default parameters, and fed with the values of $\rankingScore$ and $\aScore$ for about $6{,}550$ 
performances regularly placed in $\aSetOfPerformances$. Note that Kendall's $\tau$ is not a continuous function of $\randVarImportance$ when estimated on a finite set of performances. We designed a custom optimizer to estimate $\tau$, detailed in~\cref{sec:kendall}.

The result of our review is given in \cref{tbl:scores-review}, where the scores have been grouped in three categories: in green are the scores satisfying the three tests in all cases. In orange are those that satisfy the three tests only for fixed priors. In black are the others. We draw 4 conclusions. (1) Performance orderings satisfying our axioms can be induced by several classical scores for two-class classification (in green). (2) There exist scores that are commonly used in the literature (in orange) that cannot be used for ranking classifiers, unless the priors are fixed. (3) There exist scores that are often used to compare classifiers (\eg, the geometric mean of TNR and TPR, the markedness, Matthews Correlation Coefficient, the Odds Ratio) but that cannot be used to rank, even when the priors are fixed. (4) In all studied cases where our axioms are satisfied, there is a perfect correlation with a ranking score. This shows that our family of scores covers at least a broad part of the needs. Only one exception occurred (in gray): the accuracy achievable by chance, as considered by Cohen in the definition of his $\kappa$~\cite{Cohen1960ACoefficient}, when the classes are balanced. In this case, the score takes a constant value, that is why it satisfies our axioms.

\section{Conclusion}

We present a mathematical framework for a theory of performance-based ranking, that comes with several practical benefits. Notably, it provides a universal language to properly define tasks and characterize applications. Also, separating these concepts allows evaluating entities (algorithms, devices,~\etc) independently of application-specific preferences, as an entity can be used in various use cases, and various entities can be candidates for a same use case.

Importantly, our axioms are crucial for organizers of challenges to ensure sound rankings, to avoid, \eg, that the relative order between two methods changes when a new method appears and thus prevents drawing perennial conclusions on which is better. Our axioms are minimal requirements to satisfy, and we provide practical theorems to help check if that is the case.

Besides, practitioners often face many scores in the literature for a task, but none comes with an analysis of their usability for ranking. We prove that our theory can help find appropriate scores (we propose an infinite family of them) or legitimate existing ones. As the chosen performance score for ranking may have a major impact on the growth of a research field, our framework clarifies best practices. 

We can further deepen our axiomatic framework and infinite family of scores for ranking classifiers. In~\paperB, we particularize extensively this framework to binary classification and present the \emph{\tile}, a visualization tool that organizes these scores (among which the precision $\scorePrecision$, the true positive rate $\scoreTPR$, the true negative rate $\scoreTNR$, the scores $\scoreFBeta$, and the accuracy $\scoreAccuracy$) in a single plot. Finally, in~\paperC, we provide a comprehensive guide to using the \tile according to four practical scenarios. For that purpose, we present different \tile flavors on a real example, analyzing and ranking $74$ segmentation classifiers. We now wish to build upon this trilogy of papers to reach various research communities and impact significantly their way of establishing performance-based rankings.

\mysection{Acknowledments.}
S. Pi{\'e}rard and A. Halin are funded by the SPW EER, Wallonia, Belgium (grant 2010235, ARIAC by \href{https://www.digitalwallonia.be/en/}{DIGITALWALLONIA4.AI}); 
A. Cioppa and A. Deli{\`e}ge (grant T.0065.22) by the \href{https://www.frs-fnrs.be}{F.R.S.-FNRS}. 
\newpage
\onecolumn
\appendix
\renewcommand{\thefigure}{\Alph{section}.\arabic{subsection}.\arabic{figure}}
\setcounter{figure}{0} 

\section{Supplementary Material}
\label{sec:supplementary}

\section*{Contents}
\startcontents[mytoc]
\printcontents[mytoc]{}{0}{}

\clearpage
\subsection{List of Symbols}
\label{sec:list_of_symbols}

\subsubsection{Mathematical Symbols}
\begin{itemize}
    \item $\indicatorSymbol_U$: the 0-1 indicator function of subset $U$
    \item $\realNumbers$: the real numbers
    \item $\aRelation$: a relation
    \item $\allConvexCombinations$: the set of convex combinations
    \item $\vee$: the \emph{inclusive disjunction} (\ie, logical or)
    \item $\wedge$: the \emph{conjunction} (\ie, logical and)
    \item $\circ$: the composition of functions, \ie $(g\circ f)(x)=g(f(x))$
    \item $\expectedValueSymbol$: the mathematical expectation
\end{itemize}

\subsubsection{Symbols Related to Our Mathematical Framework}

We organize these symbols according to the 6 pillars depicted in \cref{fig:graphical_abstract}, which correspond to the 6 subsections of \cref{sec:mathematical-framework}.

\paragraph{Symbols related to the 1\textsuperscript{st} pillar (\cref{sec:mathematical-framework-performances})}
\begin{itemize}
    \item $\sampleSpace$: the sample space (universe)
    \item $\aSample$: a sample (\ie, an element of $\sampleSpace$)
    \item $\eventSpace$: the event space (a $\sigma$-algebra on $\sampleSpace$, \eg $2^\sampleSpace$)
    \item $\anEvent$: an event (\ie, an element of $\eventSpace$)
    \item $\measurableSpace$: the measurable space
    \item $\allPerformances$: all performances on $\measurableSpace$
    \item $\aSetOfPerformances$: a set of performances ($\aSetOfPerformances\subseteq\allPerformances$)
    \item $\aPerformance$: a performance (\ie, an element of $\allPerformances$)
\end{itemize}

\paragraph{Symbols related to the 2\textsuperscript{nd} pillar (\cref{sec:mathematical-framework-ordering})}
\begin{itemize}
    \item $\relWorseOrEquivalent$:  binary relation \emph{worse or equivalent} on $\allPerformances$
    \item $\relBetterOrEquivalent$: binary relation \emph{better or equivalent} on $\allPerformances$
    \item $\relEquivalent$:  binary relation \emph{equivalent} on $\allPerformances$
    \item $\relBetter$:  binary relation \emph{better} on $\allPerformances$
    \item $\relWorse$:  binary relation \emph{worse} on $\allPerformances$
    \item $\relIncomparable$: binary relation \emph{incomparable} on $\allPerformances$
\end{itemize}

\paragraph{Symbols related to the 3\textsuperscript{rd} pillar (\cref{sec:mathematical-framework-satisfaction})}
\begin{itemize}
    \item $\randVarSatisfaction$: the random variable \emph{Satisfaction}
    \item $\minSatisfaction$: the minimum satisfaction value
    \item $\maxSatisfaction$: the maximum satisfaction value
\end{itemize}

\paragraph{Symbols related to the 4\textsuperscript{th} pillar (\cref{sec:mathematical-framework-combinations})}
\begin{itemize}
    \item $\entitiesToRank$: the set of entities to rank
    \item $\anEntity$: an entity, \ie an element of $\entitiesToRank$
    \item $\evaluation$: the performance \emph{evaluation} function
    \item $\achievableByCombinations$: some performances that are for sure achievable
\end{itemize}

\paragraph{Symbols related to the 5\textsuperscript{th} pillar (\cref{sec:mathematical-framework-scores})}
\begin{itemize}
    \item $\aScore$: a score
    \item $\domainOfScore$: the domain of the score $\aScore$
    \item $\expectedValueScore{V}$: the \emph{expected value score} parameterized by the random variable $V$
    \item $\probabilisticScore{\anEvent_1\vert \anEvent_2}$: the \emph{probabilistic score} parameterized by the events $\anEvent_1$ and $\anEvent_2$
\end{itemize}

\paragraph{Symbols related to the 6\textsuperscript{th} pillar (\cref{sec:mathematical-framework-importance})}
\begin{itemize}
    \item $\randVarImportance$: the random variable \emph{Importance}
\end{itemize}

\subsubsection{Symbols used for Operations on Performances}
\begin{itemize}
    \item $\opFilter$: the \emph{filtering} operation
\end{itemize}

\subsubsection{Symbols used in the Performance Ordering and Performance-Based Ranking Theory}
\begin{itemize}
    \item $\rank$: the \emph{ranking} function, \wrt the set of entities $\entitiesToRank$
    \item $\ordering_{\aScore}$: the ordering induced by the score $\aScore$ (\cf \cref{thm:suff-cond-preorder})
    \item $\rankingScore$: the \emph{ranking score} parameterized by the importance $\randVarImportance$
    \item $\tau$: the rank correlation coefficient of Kendall~\cite{Kendall1938ANewMeasure}
\end{itemize}

\subsubsection{Symbols used for the Particular Case of Two-Class Crisp Classifications}

\paragraph{Particularization of the mathematical framework}
\begin{itemize}
    \item $\sampleTN$: the sample \emph{true negative}
    \item $\sampleFP$: the sample \emph{false positive}, \aka type I error
    \item $\sampleFN$: the sample \emph{false negative}, \aka type II error
    \item $\sampleTP$: the sample \emph{true positive}
\end{itemize}

\paragraph{Extensions to the mathematical framework}
\begin{itemize}
    \item ROC: the \emph{Receiver Operating Characteristic} space, \ie $\scoreFPR\times\scoreTPR$
    \item PR: the \emph{Precision-Recall} space, \ie $\scoreTPR\times\scorePPV$
    \item $\randVarGroundtruthClass$: the random variable for the ground truth
    \item $\randVarPredictedClass$: the random variable for the prediction
    \item $\allClasses$: the set of classes
    \item $\aClass$: a class (\ie, an element of $\allClasses$)
    \item $\classNeg$: the negative class
    \item $\classPos$: the positive class
\end{itemize}

\paragraph{Scores}
\begin{itemize}
    \item $\scoreAccuracy$: the \emph{accuracy}
    \item $\scoreTNR$: the \emph{true negative rate}
    \item $\scoreFPR$: the \emph{false positive rate}
    \item $\scoreTPR$: the \emph{true positive rate}
    \item $\scorePPV$: the \emph{positive predictive value}
    \item $\scoreFBeta$: the F-scores
    \item $\priorpos$: the \emph{prior of the positive class}
    \item $\priorneg$: the \emph{prior of the negative class}
\end{itemize}

\clearpage
\subsection{How to Use our Framework: a Little Catalog of Problems}
\label{sec:catalog-of-problems}

Throughout the paper, we have exemplified our theory with the
problem of two-class classification. This section aims at showing the universality of our theory. It presents a little catalog of other problems, together with discussions on how to use our framework for them. These discussions are introductions. As shown by \cref{sec:two-class-crisp-classification} and two recent works \cite{Pierard2024TheTile-arxiv,Halin2024AHitchhikers-arxiv}, an in-depth analysis and particularization of our theory to the various problems (\eg, to highlight their distinctive features, to review the current ranking practices from the literature and the consistency of popular scores, and to establish practical tools tailored to different user needs) may require substantial work that is out of the scope of this supplementary material.

In the following, we adopt a systematic approach: for each problem, we start by specifying a thought random experiment for the evaluation (this is an arbitrary choice since the random experiment is not unique for each problem; do not hesitate to use a different random experiment, the important is to specify it explicitly!), then we discuss the possible choices for the sample space $\sampleSpace$ (and what the set of all performances is), for the modeling of tasks with the satisfaction $\randVarSatisfaction$, for the modeling of the knowledge we have about the evaluation with the function $\achievableByCombinations$, and for the modeling of the application-specific preferences with the importance $\randVarImportance$.


\subsubsection{Multi-Class Classification (with a Note on Micro- and Macro-Averaging)}

Let us consider the following thought experiment to evaluate classifiers predicting classes in a finite and non-empty set $\allClasses$.

\global\long\def\theSource{\mathcal{S}}
\global\long\def\theSourceCodomain{\mathbb{S}}
\global\long\def\theOracle{\mathcal{O}}
\global\long\def\theOracleCodomain{\allClasses}
\global\long\def\theDescriptor{\mathcal{D}}
\global\long\def\theDescriptorCodomain{\mathbb{X}}
\global\long\def\theClassifier{\mathcal{C}}
\global\long\def\theClassifierCodomain{\allClasses}

\begin{experiment}
    (1)~Draw a sample $s\in\theSourceCodomain$ at random from a given \emph{source} $\theSource$. (2)~Apply the \emph{oracle} $\theOracle$ on $s$ to obtain the ground-truth class $y(s)\in\theOracleCodomain$. (3)~Apply a \emph{descriptor} $\theDescriptor$ on $s$ to obtain the features (\aka attributes) $x(s)\in\theDescriptorCodomain$. (4)~Feed the \emph{classifier} $\theClassifier$ with $x(s)$ to obtain the predicted class $\hat{y}(x(s))\in\theClassifierCodomain$. (5)~Set the outcome of the experiment to the pair $(y,\hat{y})\in\allClasses^2$.
\end{experiment}

\begin{description}

\item [Choice for $\sampleSpace$ and $\allPerformances$.] Our theory applies with $\sampleSpace=\allClasses^2$. By definition, the function $\evaluation$ gives, for any classifier $\theClassifier:\theDescriptorCodomain\rightarrow\theClassifierCodomain$ (the evaluated entity, either deterministic or not), the distribution of outcomes resulting from this random experiment: $\aPerformance_{\theClassifier}=\evaluation(\theClassifier)$. Note that it is implicit that the performances are specific for some given source $\theSource$ (\eg, evaluation dataset), oracle $\theOracle$, and descriptor $\theDescriptor$. By convenience, one can manipulate the ground-truth and predicted classes with, respectively, the random variables $\randVarGroundtruthClass$ and $\randVarPredictedClass$ defined in such a way that $\aSample=(\randVarGroundtruthClass(\aSample),\randVarPredictedClass(\aSample))\,\forall \aSample\in\sampleSpace$.

\item[Choice for $\randVarSatisfaction$.] Several classification tasks can be distinguished, as the following two examples show. (1) One can consider that all erroneous classifications are unsatisfactory and that correct classifications are satisfactory. For this task, the satisfaction is then binary and given by $\randVarSatisfaction=\indicatorSymbol_{\randVarGroundtruthClass=\randVarPredictedClass}$. The expected value of the satisfaction, which is a particular case of ranking scores, is then equal to the multi-class accuracy. (2) One can also consider the similarity $\textrm{sim}:\allClasses^2\rightarrow\realNumbers$ between classes, and choose the satisfaction accordingly: $\randVarSatisfaction(\aSample)=\textrm{sim}(\randVarGroundtruthClass(\aSample);\randVarPredictedClass(\aSample))$. A wide variety of tasks can be considered by tuning $\textrm{sim}$. In general, the expected value of the satisfaction is different from the multi-class accuracy.

\item[Choice for $\achievableByCombinations$.]  As the classifier is used once and only once during the execution of the evaluation, we know that if the performances $\aPerformance_1$ and $\aPerformance_2$ are achievable by some classifiers $\theClassifier_1$ and $\theClassifier_2$, then any performance $\overline{\aPerformance}=\lambda_1\aPerformance_1+\lambda_2\aPerformance_2$ (with $\lambda_1\ge0$, $\lambda_2\ge0$, and $\lambda_1+\lambda_2=1$) is achievable by a classifier $\overline{\theClassifier}$ obtained by a non-deterministic combination of $\theClassifier_1$ and $\theClassifier_2$ that chooses them with respective probabilities $\lambda_1$ and $\lambda_2$. Thus, $\achievableByCombinations=\allConvexCombinations$ makes sense, and all ranking scores can be used to rank classifiers. However, it would be possible to go further, by considering other functions $\achievableByCombinations$ that would include the knowledge that we can predict the performance achievable by composing the classifier with any of the $|\allClasses|^{|\allClasses|}$ functions $f:\allClasses\rightarrow\allClasses$. This would lead to other performance orderings suitable for ranking classifiers.

\item[Choice for $\randVarImportance$.] When $\achievableByCombinations=\allConvexCombinations$, we have demonstrated that all rankings induced by the ranking scores $\rankingScore$ satisfy all our three axioms. This leaves a great flexibility for the users to fine-tune the ranking \wrt their application-specific preferences through the random variable $\randVarImportance$. As we have seen, the only constraints are that $\randVarImportance\ne0$ and $\randVarImportance(\aSample)\ge0\,\forall\aSample\in\sampleSpace$.

\item[Note on micro- and macro-averaging.] Micro- and macro-averaging are commonly used techniques to build scores for multi-class classification from scores for two-class classification \cite{Sokolova2009ASystematic}. We warn that they have pitfalls. In general, micro- and macro-averaging scores suitable for ranking two-class classifiers do not lead to scores suitable for ranking multi-class classifiers. The accuracy put aside, the performance orderings induced from the micro-averaged versions of the scores put in green in \cref{tbl:scores-review} are incompatible with $\randVarSatisfaction=\indicatorSymbol_{\randVarGroundtruthClass=\randVarPredictedClass}$: our \second axiom is not satisfied. Moreover, the accuracy put again aside, the performance orderings induced from the macro-averaged versions of the scores put in green in \cref{tbl:scores-review} are incompatible with $\achievableByCombinations=\allConvexCombinations$: our \third axiom is not satisfied. A solution consists in using directly ranking scores defined for multi-class classification.

\end{description}


\subsubsection{Regression (with a Note on the Mean Squared Error and the Mean Absolute Error)}

Let us consider the following thought experiment to evaluate regressors.

\global\long\def\theOracleCodomain{\realNumbers}
\global\long\def\theRegressor{\mathcal{R}}
\global\long\def\theRegressorCodomain{\realNumbers}

\begin{experiment}
    (1)~Draw a sample $s\in\theSourceCodomain$ at random from a given \emph{source} $\theSource$. (2)~Apply the \emph{oracle} $\theOracle$ on $s$ to obtain the ground-truth value $y(s)\in\theOracleCodomain$. (3)~Apply a \emph{descriptor} $\theDescriptor$ on $s$ to obtain the features (\aka attributes) $x(s)\in\theDescriptorCodomain$. (4)~Feed the \emph{regressor} $\theRegressor$ with $x(s)$ to obtain the predicted value $\hat{y}(x(s))\in\theRegressorCodomain$. (5)~Set the outcome of the experiment to the pair $(y,\hat{y})\in\realNumbers^2$.
\end{experiment}

\global\long\def\randVarGroundtruthValue{Y}
\global\long\def\randVarPredictedValue{\hat{Y}}

\begin{description}

\item [Choice for $\sampleSpace$ and $\allPerformances$.] Our theory applies with $\sampleSpace=\realNumbers^2$. By definition, the function $\evaluation$ gives, for any regressor $\theRegressor:\theDescriptorCodomain\rightarrow\theRegressorCodomain$ (the evaluated entity, either deterministic or not), the distribution of outcomes resulting from this random experiment. It is the performance $\aPerformance_{\theRegressor}=\evaluation(\theRegressor)$ of $\theRegressor$. By convenience, one can manipulate the ground-truth and predicted values with, respectively, the random variables $\randVarGroundtruthValue$ and $\randVarPredictedValue$ defined in such a way that $\aSample=(\randVarGroundtruthValue(\aSample),\randVarPredictedValue(\aSample))\,\forall \aSample\in\sampleSpace$.

\item[Choice for $\randVarSatisfaction$.] Clearly, it is not a good idea to choose $\randVarSatisfaction=\indicatorSymbol_{\randVarGroundtruthValue=\randVarPredictedValue}$, similarly as one can do in classification. In practice, a regressor as no chance to predict the same value as the oracle, so this unfortunate choice for $\randVarSatisfaction$ would lead to performances $\aPerformance$ such that $\aPerformance(\randVarSatisfaction=0)=1$. In other words, all performances observed about real regressors would belong to the set of the worst performances (see \cref{corollary:worst-performances}), and their ranking would be of little interest (see \cref{corollary:quivalent-incomparable-performances}). A better option consists in specifying a tolerance $\epsilon>0$ and choosing $\randVarSatisfaction=\indicatorSymbol_{|\randVarGroundtruthValue-\randVarPredictedValue|\le\epsilon}$. An even more flexible option, which takes advantage of the fact that satisfaction values do not necessarily have to be positive, is to choose $\randVarSatisfaction=f(|\randVarGroundtruthValue-\randVarPredictedValue|)$ with any arbitrarily chosen monotonically decreasing function $f$. The plethora of choices that can be made for $\randVarSatisfaction$ makes it clear that there is an infinity of tasks related to the regression problem.

\item[Choice for $\achievableByCombinations$.] As the regressor is used once and only once during the execution of the evaluation, we know that if the performances $\aPerformance_1$ and $\aPerformance_2$ are achievable by some regressors $\theRegressor_1$ and $\theRegressor_2$, then any performance $\overline{\aPerformance}=\lambda_1\aPerformance_1+\lambda_2\aPerformance_2$ (with $\lambda_1\ge0$, $\lambda_2\ge0$, and $\lambda_1+\lambda_2=1$) is achievable by a regressor $\overline{\theRegressor}$ obtained by a non-deterministic combination of $\theRegressor_1$ and $\theRegressor_2$ that chooses them with respective probabilities $\lambda_1$ and $\lambda_2$. Thus, $\achievableByCombinations=\allConvexCombinations$ makes sense, and all ranking scores can be used to rank regressors. However, it would be possible to go further, by considering other functions $\achievableByCombinations$ that would include the knowledge that we can predict the performance achievable by adding noise, or applying a transformation on the output of the regressor. This would lead to other performance orderings suitable for ranking regressors.

\item[Choice for $\randVarImportance$.] When $\achievableByCombinations=\allConvexCombinations$, we have demonstrated that all rankings induced by the ranking scores $\rankingScore$ satisfy all our three axioms. This leaves a great flexibility for the users to fine-tune the ranking \wrt their application-specific preferences through the random variable $\randVarImportance$. As we have seen, the only constraints are that $\randVarImportance\ne0$ and $\randVarImportance(\aSample)\ge0\,\forall\aSample\in\sampleSpace$.

\item[Note on the mean squared error and the mean absolute error.] 
If we choose $\randVarSatisfaction=-|\randVarGroundtruthValue-\randVarPredictedValue|^2$, the ranking score $\rankingScore$ corresponding to uniform importance values, \ie, the expected satisfaction $\scoreExpectedSatisfaction$, yields a ranking that minimizes the \emph{mean squared error} (MSE). If we choose $\randVarSatisfaction=-|\randVarGroundtruthValue-\randVarPredictedValue|$, the ranking score $\rankingScore$ corresponding to uniform importance values, \ie, the expected satisfaction $\scoreExpectedSatisfaction$, yields a ranking that minimizes the \emph{mean absolute error} (MAE).

\end{description}


\subsubsection{Information Retrieval}

\global\long\def\allQueries{\mathbb{Q}}
\global\long\def\resultsSetGroundtruth{\mathbb{Y}}
\global\long\def\resultsSetPredicted{\hat{\mathbb{Y}}}

Let us consider the following thought experiment to evaluate information retrieval systems. We denote by $\allQueries$ the set of all possible queries.

\global\long\def\theSystem{\mathcal{S}}

\begin{experiment}
    (1) Draw a query $q\in\allQueries$ at random from a given source $\theSource$. (2) Apply the oracle $\theOracle$ on $q$ to obtain the ground-truth set of results $\resultsSetGroundtruth$. (3) Apply the evaluated information retrieval system $\theSystem$ on $q$ to obtain the predicted set of results $\resultsSetPredicted$. (4) If $\resultsSetGroundtruth=\emptyset$ and $\resultsSetPredicted=\emptyset$, restart the experiment, otherwise draw a result $r$ at random in $\resultsSetGroundtruth\cup\resultsSetPredicted$. (5) Choose the outcome as follows: $\sampleFP$ if $r\notin\resultsSetGroundtruth$ and $r\in\resultsSetPredicted$, $\sampleFN$ if $r\in\resultsSetGroundtruth$ and $r\notin\resultsSetPredicted$, and $\sampleTP$ if $r\in\resultsSetGroundtruth$ and $r\in\resultsSetPredicted$.
\end{experiment}

\begin{description}

\item [Choice for $\sampleSpace$ and $\allPerformances$.] Our theory applies with $\sampleSpace=\{\sampleFP,\sampleFN,\sampleTP\}$. By definition, the function $\evaluation$ gives, for any retrieval system $\theSystem$ defined on $\allQueries$ (the evaluated entity, either deterministic or not), the distribution of outcomes resulting from this random experiment: $\aPerformance_{\theSystem}=\evaluation(\theSystem)$.

\item[Choice for $\randVarSatisfaction$.] Intuitively, everyone certainly agrees that $\randVarSatisfaction(\sampleFP)<\randVarSatisfaction(\sampleTP)$ and $\randVarSatisfaction(\sampleFN)<\randVarSatisfaction(\sampleTP)$. But we expect different opinions regarding whether the outcome (sample) $\sampleFP$ gives less, equal, or more satisfaction than $\sampleFN$.

\item[Choice for $\achievableByCombinations$.] This random experiment is very interesting as, during its execution, the evaluated entity (the information retrieval system $\theSystem$) can be used multiple times. In such a case, we have to discuss whether $\achievableByCombinations=\allConvexCombinations$ is  adequate. Let us consider two systems $\theSystem_1$, $\theSystem_2$ and their respective performances $\aPerformance_1=\evaluation(\theSystem_1)$, $\aPerformance_2=\evaluation(\theSystem_2)$. It is possible to show that the performance of a retrieval system $\overline{\theSystem}$ obtained by a non-deterministic combination of $\theSystem_1$ and $\theSystem_2$, that chooses them with respective probabilities $\lambda_1$ and $\lambda_2$, is some interpolated performance $\overline{\aPerformance}=\mu_1\aPerformance_1+\mu_2\aPerformance_2$ (with $\mu_1\ge0$, $\mu_2\ge0$, and $\mu_1+\mu_2=1$). Unless being in very particular cases, $\mu\ne\lambda$. In other words, we know that the performances that are convex combinations of achievable performances are also achievable (for any $\lambda$, there exists $\mu$), but we do not know in general how to achieve them (for most $\mu$ it is not possible to determine $\lambda$). This contrasts with the other kinds of problems discussed in this catalog. In fact, the question we raise here is not specific to the information retrieval problem: it is peculiar to the random experiment that we have chosen for it. By slightly modifying the thought experiment, the question vanishes: instead of restarting the experiment when $\resultsSetGroundtruth\cup\resultsSetPredicted=\emptyset$, one could yield a fourth outcome (and add it to the sample space $\sampleSpace$). By doing so, the evaluated entity $\theSystem$ is used only once and $\achievableByCombinations=\allConvexCombinations$ makes sense for sure.

\item[Choice for $\randVarImportance$.] If $\achievableByCombinations=\allConvexCombinations$ is considered as adequate, then we have demonstrated that all rankings induced by the ranking scores $\rankingScore$ satisfy all our three axioms. This leaves a great flexibility for the users to fine-tune the ranking \wrt their application-specific preferences through the random variable $\randVarImportance$. As we have seen, the only constraints are that $\randVarImportance\ne0$ and $\randVarImportance(\aSample)\ge0\,\forall\aSample\in\sampleSpace$.

\end{description}


\subsubsection{Detection (with a Note about the Intersection-over-Union and the F-Score)}

Different types of detections are present in the literature. An example of spatial detection aims at predicting the axis-aligned bounding boxes around all the objects that match some given properties (\ie, a semantic class) in input images. Examples of temporal detections include the detection of events in video streams and in audio recordings. By definition, such detection problems are called \emph{action spotting} when the temporal window is small, and \emph{activity detection} otherwise. Let us consider the following, generic, thought experiment to evaluate detectors.

\global\long\def\theDetector{\mathcal{D}}
\global\long\def\sampleNothing{\logof}
\global\long\def\scoreIoU{IoU}

\begin{experiment}
    (1) Draw an input at random from a given \emph{source} $\theSource$ (\eg, dataset). (2) Apply the \emph{oracle} $\theOracle$ on it to obtain a set $\resultsSetGroundtruth$ of ground-truth detections. (3) Also apply the \emph{detector} $\theDetector$ on it to obtain a set $\resultsSetPredicted$ of predicted detections. (4) If $\resultsSetGroundtruth=\emptyset$ and $\resultsSetPredicted=\emptyset$, then end the experiment with the outcome $\sampleNothing$. Otherwise: (5) Apply a matching criterion between $\resultsSetGroundtruth$ and $\resultsSetPredicted$ such that to any detection in $\resultsSetGroundtruth$ should be associated at most a detection in $\resultsSetPredicted$ and vice versa. (6) Draw a detection $d$ at random in $\resultsSetGroundtruth\cup\resultsSetPredicted$. (7) Give as outcome $\sampleFP$, $\sampleFN$ or $\sampleTP$ depending on whether $d$ is a prediction, a ground truth, or both (\ie, a match).
\end{experiment}

\begin{description}

\item [Choice for $\sampleSpace$ and $\allPerformances$.] Our theory applies with $\sampleSpace=\{\sampleNothing,\sampleFP,\sampleFN,\sampleTP\}$. By definition, the function $\evaluation$ gives, for any detector (the evaluated entity), the distribution of outcomes resulting from this random experiment. It is the performance $\aPerformance_{\theDetector}=\evaluation(\theDetector)$ of $\theDetector$.

\item[Choice for $\randVarSatisfaction$.] Intuitively, everyone certainly agrees that $\sampleNothing$ and $\sampleTP$ give entire satisfaction. Moreover, we expect agreement on $\randVarSatisfaction(\sampleFP)<\randVarSatisfaction(\sampleTP)$ and $\randVarSatisfaction(\sampleFN)<\randVarSatisfaction(\sampleTP)$. But we expect different opinions regarding whether $\sampleFP$ gives less, equal, or more satisfaction than $\sampleFN$.

\item[Choice for $\achievableByCombinations$.] As the detector is used once and only once during the execution of the evaluation, we know that if the performances $\aPerformance_1$ and $\aPerformance_2$ are achievable by some detectors $\theDetector_1$ and $\theDetector_2$, then any performance $\overline{\aPerformance}=\lambda_1\aPerformance_1+\lambda_2\aPerformance_2$ (with $\lambda_1\ge0$, $\lambda_2\ge0$, and $\lambda_1+\lambda_2=1$) is achievable by a detector $\overline{\theDetector}$ obtained by a non-deterministic combination of $\theDetector_1$ and $\theDetector_2$ that chooses them with respective probabilities $\lambda_1$ and $\lambda_2$. Thus, $\achievableByCombinations=\allConvexCombinations$ makes sense, and all ranking scores can be used to rank detectors.

\item[Choice for $\randVarImportance$.]  If $\achievableByCombinations=\allConvexCombinations$ is considered as adequate, then we have demonstrated that all rankings induced by the ranking scores $\rankingScore$ satisfy all our three axioms. This leaves a great flexibility for the users to fine-tune the ranking \wrt their application-specific preferences through the random variable $\randVarImportance$. As we have seen, the only constraints are that $\randVarImportance\ne0$ and $\randVarImportance(\aSample)\ge0\,\forall\aSample\in\sampleSpace$.

\item[Note about the Intersection-over-Union and the F-score.] Traditionally, in the literature, as soon as one has symbols $\sampleFP$, $\sampleFN$, and $\sampleTP$, regardless of their very fine meaning, one defines quantities $\scoreIoU=\frac{\aPerformance(\sampleTP)}{\aPerformance(\sampleFP)+\aPerformance(\sampleFN)+\aPerformance(\sampleTP)}$ and $\scoreFOne=\frac{2\aPerformance(\sampleTP)}{\aPerformance(\sampleFP)+\aPerformance(\sampleFN)+2\aPerformance(\sampleTP)}$, and name them \emph{Intersection-over-Union} (or \emph{Jaccard}) and \emph{F-one}, respectively. The exact meaning of these quantities is not well standardized. In particular, the random experiment supporting the evaluation, if it exists, is rarely specified explicitly. For this reason, we cannot give the guarantee that these quantities are suitable to rank detectors in all works in which they have been used. However, with the random experiment given here-above, with $\achievableByCombinations=\allConvexCombinations$, and with $\randVarSatisfaction=\indicatorSymbol_{\{\sampleNothing,\sampleTP\}}$, we can guarantee that $\scoreIoU$ and $\scoreFOne$ are suitable to rank detectors because they are equal to the ranking scores with, respectively, $\randVarImportance=\indicatorSymbol_{\{\sampleFP,\sampleFN,\sampleTP\}}$ and $\randVarImportance= \indicatorSymbol_{\{\sampleFP,\sampleTP\}} + \indicatorSymbol_{\{\sampleFN,\sampleTP\}}$. Thus, the performance orderings induced by them fulfill our three axioms.

\end{description}


\subsubsection{Clustering (with a Note about Fowlkes-Mallows Index)}

\global\long\def\theClusterer{\mathcal{C}}

Let us consider the following thought experiment to evaluate clustering methods. These methods aim to place in different clusters (groups) dissimilar objects and in the same cluster (group) objects that are similar to each other. We denote by $\allEntities$ the set of elements that these methods have to deal with. For the sake of simplicity, we do not consider hierarchical clustering.

\begin{experiment}
    (1) Apply both the \emph{clustering method} $\theClusterer$ and the \emph{oracle} $\theOracle$ on $\allEntities$ to obtain, respectively, the predicted and ground-truth clusterings. (2) Randomly draw two distinct elements, $\anEntity_1$ and $\anEntity_2$, from $\allEntities$. (3) Consider that the pair $(\anEntity_1;\anEntity_2)$ is a negative or a positive, in a given clustering, when $\anEntity_1$ and $\anEntity_2$ are in different clusters or in the same cluster, respectively. (4) Choose the outcome as follows: $\sampleTN$ when $(\anEntity_1;\anEntity_2)$ is negative in both the predicted and ground-truth clusterings, $\sampleFP$ when $(\anEntity_1;\anEntity_2)$ is negative in the ground-truth clustering and positive in the predicted clustering, $\sampleFN$ when $(\anEntity_1;\anEntity_2)$ is positive in the ground-truth clustering and negative in the predicted clustering, and $\sampleTP$ when $(\anEntity_1;\anEntity_2)$ is positive in both the predicted and ground-truth clusterings.
\end{experiment}

\global\long\def\scoreFMI{FMI}%

\begin{description}

\item [Choice for $\sampleSpace$ and $\allPerformances$.] Our theory applies with $\sampleSpace=\{\sampleTN,\sampleFP,\sampleFN,\sampleTP\}$. By definition, the function $\evaluation$ gives, for any clustering method (the evaluated entity), the distribution of outcomes resulting from this random experiment. It is the performance $\aPerformance_{\theClusterer}=\evaluation(\theClusterer)$ of $\theClusterer$.

\item[Choice for $\randVarSatisfaction$.] When $\randVarSatisfaction$ is chosen such that $\randVarSatisfaction(\sampleFP)=\randVarSatisfaction(\sampleFN)=0$ and $\randVarSatisfaction(\sampleTN)=\randVarSatisfaction(\sampleTP)=1$, we are in the same setting as the one we studied for the two-class classification in \cref{sec:two-class-crisp-classification}. This is indeed not because we use the same symbols for the elements of $\sampleSpace$ ---this is just a convention---, but because, in both settings, we have $|\sampleSpace|=4$, $|\randVarSatisfaction=0|=2$, and $|\randVarSatisfaction=1|=2$. This implies that the performance orderings that satisfy our three axioms for ranking two-class classifiers can also be used for ranking clustering methods, and vice versa.

\item[Choice for $\achievableByCombinations$.] As the clustering method is used once and only once during the execution of the evaluation, we know that if the performances $\aPerformance_1$ and $\aPerformance_2$ are achievable by some clustering methods $\theClusterer_1$ and $\theClusterer_2$, then any performance $\overline{\aPerformance}=\lambda_1\aPerformance_1+\lambda_2\aPerformance_2$ (with $\lambda_1\ge0$, $\lambda_2\ge0$, and $\lambda_1+\lambda_2=1$) is achievable by a clustering method $\overline{\theClusterer}$ obtained by a non-deterministic combination of $\theClusterer_1$ and $\theClusterer_2$ that chooses them with respective probabilities $\lambda_1$ and $\lambda_2$. Thus, $\achievableByCombinations=\allConvexCombinations$ makes sense, and all ranking scores can be used to rank clustering methods.

\item[Choice for $\randVarImportance$.] When $\achievableByCombinations=\allConvexCombinations$, we have demonstrated that all rankings induced by the ranking scores $\rankingScore$ satisfy all our three axioms. This leaves a great flexibility for the users to fine-tune the ranking \wrt their application-specific preferences through the random variable $\randVarImportance$. As we have seen, the only constraints are that $\randVarImportance\ne0$ and $\randVarImportance(\aSample)\ge0\,\forall\aSample\in\sampleSpace$.

\item[Note about Fowlkes-Mallows index.] The score $\scoreFMI$ known as \emph{Fowlkes-Mallows index}~\cite{Fowlkes1983AMethod} and \emph{cosine coefficient}~\cite{Ballabio2018Multivariate}, which is commonly used for clustering methods and defined as the geometric mean of the positive predictive value $\scorePPV$ (\aka precision) and true positive rate $\scoreTPR$ (\aka sensitivity and recall), does not satisfy our \third axiom: the performance ordering induced by $\scoreFMI$ is incompatible with $\achievableByCombinations=\allConvexCombinations$. More precisely, the clustering method $\theClusterer$ obtained by randomly choosing between some methods $\theClusterer_1$ or $\theClusterer_2$ can be such that $\scoreFMI(\evaluation(\theClusterer))<\min(\scoreFMI(\evaluation(\theClusterer_1)),\scoreFMI(\evaluation(\theClusterer_2)))$, while it makes no sense to say that $\theClusterer$ can be worse than $\theClusterer_1$ or $\theClusterer_2$. From this perspective, it is advisable to use any ranking score instead of $\scoreFMI$, for example, those in green in \cref{tbl:scores-review}.

\end{description}


\subsubsection{Ranking (with a Note about Kendall's $\tau$)}

Let us consider the following thought experiment to evaluate ranking methods. We denote by $\allEntities$ the set of elements that these methods have to rank. For the sake of simplicity, we prefer to deal only with the case with no tie hereafter.

\global\long\def\theRanker{\mathcal{R}}

\begin{experiment}
    (1) Apply both the \emph{ranking method} $\theRanker$ and the \emph{oracle} $\theOracle$ on $\allEntities$ to obtain, respectively, the predicted and ground-truth sequences of elements. (2) Randomly draw two distinct elements, $\anEntity_1$ and $\anEntity_2$, from $\allEntities$. (3) Four cases can occur depending on whether $\anEntity_1$ is before or after $\anEntity_2$ in the predicted sequence and whether $\anEntity_1$ is before or after $\anEntity_2$ in the ground-truth sequence. Nevertheless, two outcomes are enough: choose $\smiley$ if $\anEntity_1$ and $\anEntity_2$ appear in the same order in both sequences, $\frownie$ otherwise.
\end{experiment}

\begin{description}

\item [Choice for $\sampleSpace$ and $\allPerformances$.] Our theory applies with $\sampleSpace=\{\smiley,\frownie\}$. By definition, the function $\evaluation$ gives, for any ranking method (the evaluated entity), the distribution of outcomes resulting from this random experiment. It is the performance $\aPerformance_{\theRanker}=\evaluation(\theRanker)$ of $\theRanker$. The probability of drawing a \emph{discordant pair} is given by $\aPerformance(\{\frownie\})$, and the probability of drawing a \emph{concordant pair} is given by $\aPerformance(\{\smiley\})$.

\item[Choice for $\randVarSatisfaction$.] Clearly, $\randVarSatisfaction(\frownie)<\randVarSatisfaction(\smiley)$ is wanted.

\item[Choice for $\achievableByCombinations$.] As the ranking method is used once and only once during the execution of the evaluation, we know that if the performances $\aPerformance_1$ and $\aPerformance_2$ are achievable by some ranking methods $\theRanker_1$ and $\theRanker_2$, then any performance $\overline{\aPerformance}=\lambda_1\aPerformance_1+\lambda_2\aPerformance_2$ (with $\lambda_1\ge0$, $\lambda_2\ge0$, and $\lambda_1+\lambda_2=1$) is achievable by a ranking method $\overline{\theRanker}$ obtained by a non-deterministic combination of $\theRanker_1$ and $\theRanker_2$ that chooses them with respective probabilities $\lambda_1$ and $\lambda_2$. Thus, $\achievableByCombinations=\allConvexCombinations$ makes sense, and all ranking scores can be used to rank ranking methods.

\item[Choice for $\randVarImportance$.] Because $|\sampleSpace|=2$ and $\randVarSatisfaction(\frownie)\ne\randVarSatisfaction(\smiley)$, we are in a particular case in which all ranking scores rank the ranking methods in the same way. From this point of view, fine-tuning $\randVarImportance$ is useless.

\item[Note about Kendall's $\tau$.] When $\randVarSatisfaction(\frownie)=-1$ and $\randVarSatisfaction(\smiley)=1$, the expected value of the satisfaction is given by $\expectedValueScore{\randVarSatisfaction}(\aPerformance)=1-2\aPerformance(\{\frownie\})=\aPerformance(\{\smiley\})-\aPerformance(\{\frownie\})=\tau(\aPerformance)$. In other words, with the task corresponding to this choice for the satisfaction, Kendall's correlation coefficient $\tau$ \cite{Kendall1938ANewMeasure} is the ranking score corresponding to uniform importance values.

\end{description}

\clearpage
\subsection{Supplementary Material about \cref{sec:mathematical-framework-ordering}}

This section is devoted to reminders about the order theory.

\subsubsection{Reminders of Classical Definitions.}

Let $\aRelation$ be a homogeneous binary relation on $\allPerformances$.
It is said:
\begin{itemize}
\item \emph{reflexive} \iif $\aPerformance\aRelation\aPerformance\,\forall\aPerformance$;
\item \emph{irreflexive} \iif $\nexists\aPerformance:\aPerformance\aRelation\aPerformance$;
\item \emph{transitive} \iif $\aPerformance_{1}\aRelation\aPerformance_{2}\wedge\aPerformance_{2}\aRelation\aPerformance_{3}\Rightarrow\aPerformance_{1}\aRelation\aPerformance_{3}\,\forall\aPerformance_{1},\aPerformance_{2},\aPerformance_{3}$;
\item \emph{symmetric} \iif $\aPerformance_{1}\aRelation\aPerformance_{2}\Leftrightarrow\aPerformance_{2}\aRelation\aPerformance_{1}\,\forall\aPerformance_{1},\aPerformance_{2}$;
\item \emph{asymmetric} \iif $\nexists(\aPerformance_{1},\aPerformance_{2}):\aPerformance_{1}\aRelation\aPerformance_{2}\wedge\aPerformance_{2}\aRelation\aPerformance_{1}$;
\item and \emph{antisymmetric} \iif $\aPerformance_{1}\aRelation\aPerformance_{2}\wedge\aPerformance_{2}\aRelation\aPerformance_{1}\Rightarrow\aPerformance_{1}=\aPerformance_{2}$.
\end{itemize}
Two homogeneous binary relations $\aRelation_{a}$ and $\aRelation_{b}$
on $\allPerformances$ are said\emph{ converse }\iif $\aPerformance_{1}\aRelation_{a}\aPerformance_{2}\Leftrightarrow\aPerformance_{2}\aRelation_{b}\aPerformance_{1}\,\forall\aPerformance_{1},\aPerformance_{2}$.

A relation $\aRelation$ is:
\begin{itemize}
\item an \emph{equivalence }\iif it is reflexive, transitive, and symmetric;
\item a \emph{preorder} \iif it is reflexive and transitive;
\item and an \emph{order} \iif it is reflexive, transitive, and antisymmetric.
\end{itemize}
An order $\aRelation$ is said \emph{total} \iif $\nexists(\aPerformance_{1},\aPerformance_{2}):\aPerformance_{1}\not\aRelation\aPerformance_{2}\wedge\aPerformance_{2}\not\aRelation\aPerformance_{1}$. It is said \emph{partial} otherwise.

\subsubsection{The 4 Cases in the Comparison of Two Performances with a Preorder
$\protect\relWorseOrEquivalent$.}

Let us now consider a preorder $\relWorseOrEquivalent$ and derive
the homogeneous binary relations $\relEquivalent$, $\relBetter$,
$\relWorse$, $\relIncomparable$ as follows:

\begin{align}
\aPerformance_{1}\relEquivalent\aPerformance_{2} & \Leftrightarrow\aPerformance_{1}\relWorseOrEquivalent\aPerformance_{2}\wedge\aPerformance_{2}\relWorseOrEquivalent\aPerformance_{1}\label{eq:rel-equivalent}\\
\aPerformance_{1}\relBetter\aPerformance_{2} & \Leftrightarrow\aPerformance_{1}\not\relWorseOrEquivalent\aPerformance_{2}\wedge\aPerformance_{2}\relWorseOrEquivalent\aPerformance_{1}\label{eq:rel-better}\\
\aPerformance_{1}\relWorse\aPerformance_{2} & \Leftrightarrow\aPerformance_{1}\relWorseOrEquivalent\aPerformance_{2}\wedge\aPerformance_{2}\not\relWorseOrEquivalent\aPerformance_{1}\label{eq:rel-worse}\\
\aPerformance_{1}\relIncomparable\aPerformance_{2} & \Leftrightarrow\aPerformance_{1}\not\relWorseOrEquivalent\aPerformance_{2}\wedge\aPerformance_{2}\not\relWorseOrEquivalent\aPerformance_{1}\,.\label{eq:rel-incomparable}
\end{align}

Indeed, we have:
\begin{align}
\aPerformance_{1}\relWorseOrEquivalent\aPerformance_{2} & \Leftrightarrow\aPerformance_{1}\relWorse\aPerformance_{2}\vee\aPerformance_{1}\relEquivalent\aPerformance_{2}\,.\label{eq:rel-worse-or-equivalent}
\end{align}
Similarly, one can derive other binary relations taking unions of
$\relEquivalent$, $\relBetter$, $\relWorse$, or $\relIncomparable$.
For example, 

\begin{align}
\aPerformance_{1}\relBetterOrEquivalent\aPerformance_{2} & \Leftrightarrow\aPerformance_{1}\relBetter\aPerformance_{2}\vee\aPerformance_{1}\relEquivalent\aPerformance_{2}\,.\label{eq:rel-better-or-equivalent}
\end{align}

\subsubsection{Implications of the Transitivity of $\protect\relWorseOrEquivalent$.}

We can easily check, for each $\aRelation_{ab}\in\{\relWorseOrEquivalent,\not\relWorseOrEquivalent\}$,
each $\aRelation_{ba}\in\{\relWorseOrEquivalent,\not\relWorseOrEquivalent\}$,
each $\aRelation_{bc}\in\{\relWorseOrEquivalent,\not\relWorseOrEquivalent\}$,
each $\aRelation_{cb}\in\{\relWorseOrEquivalent,\not\relWorseOrEquivalent\}$,
each $\aRelation_{ca}\in\{\relWorseOrEquivalent,\not\relWorseOrEquivalent\}$,
and each $\aRelation_{ac}\in\{\relWorseOrEquivalent,\not\relWorseOrEquivalent\}$,
if there exists $(\aPerformance_{a},\aPerformance_{b},\aPerformance_{c})$
such that $\aPerformance_{a}\aRelation_{ab}\aPerformance_{b}$, $\aPerformance_{b}\aRelation_{ba}\aPerformance_{a}$,
$\aPerformance_{b}\aRelation_{bc}\aPerformance_{c}$, $\aPerformance_{c}\aRelation_{cb}\aPerformance_{b}$,
$\aPerformance_{c}\aRelation_{ca}\aPerformance_{a}$, and $\aPerformance_{a}\aRelation_{ac}\aPerformance_{c}$.
Because of the assumed transitivity of $\relWorseOrEquivalent$, there
are only $29$ possible cases out of the $2^{6}$:

\begin{multicols}{2}
\begin{enumerate}
\item $\aPerformance_a \relIncomparable \aPerformance_b$, $\aPerformance_b \relIncomparable \aPerformance_c$, $\aPerformance_a \relIncomparable \aPerformance_c$
\item $\aPerformance_a \relIncomparable \aPerformance_b$, $\aPerformance_b \relIncomparable \aPerformance_c$, $\aPerformance_a \relBetter \aPerformance_c$
\item $\aPerformance_a \relIncomparable \aPerformance_b$, $\aPerformance_b \relIncomparable \aPerformance_c$, $\aPerformance_a \relWorse \aPerformance_c$
\item $\aPerformance_a \relIncomparable \aPerformance_b$, $\aPerformance_b \relIncomparable \aPerformance_c$, $\aPerformance_a \relEquivalent \aPerformance_c$
\item $\aPerformance_a \relIncomparable \aPerformance_b$, $\aPerformance_b \relBetter \aPerformance_c$, $\aPerformance_a \relIncomparable \aPerformance_c$
\item $\aPerformance_a \relIncomparable \aPerformance_b$, $\aPerformance_b \relBetter \aPerformance_c$, $\aPerformance_a \relBetter \aPerformance_c$
\item $\aPerformance_a \relIncomparable \aPerformance_b$, $\aPerformance_b \relWorse \aPerformance_c$, $\aPerformance_a \relIncomparable \aPerformance_c$
\item $\aPerformance_a \relIncomparable \aPerformance_b$, $\aPerformance_b \relWorse \aPerformance_c$, $\aPerformance_a \relWorse \aPerformance_c$
\item $\aPerformance_a \relIncomparable \aPerformance_b$, $\aPerformance_b \relEquivalent \aPerformance_c$, $\aPerformance_a \relIncomparable \aPerformance_c$
\item $\aPerformance_a \relBetter \aPerformance_b$, $\aPerformance_b \relIncomparable \aPerformance_c$, $\aPerformance_a \relIncomparable \aPerformance_c$
\item $\aPerformance_a \relBetter \aPerformance_b$, $\aPerformance_b \relIncomparable \aPerformance_c$, $\aPerformance_a \relBetter \aPerformance_c$
\item $\aPerformance_a \relBetter \aPerformance_b$, $\aPerformance_b \relBetter \aPerformance_c$, $\aPerformance_a \relBetter \aPerformance_c$
\item $\aPerformance_a \relBetter \aPerformance_b$, $\aPerformance_b \relWorse \aPerformance_c$, $\aPerformance_a \relIncomparable \aPerformance_c$
\item $\aPerformance_a \relBetter \aPerformance_b$, $\aPerformance_b \relWorse \aPerformance_c$, $\aPerformance_a \relBetter \aPerformance_c$
\item $\aPerformance_a \relBetter \aPerformance_b$, $\aPerformance_b \relWorse \aPerformance_c$, $\aPerformance_a \relWorse \aPerformance_c$
\item $\aPerformance_a \relBetter \aPerformance_b$, $\aPerformance_b \relWorse \aPerformance_c$, $\aPerformance_a \relEquivalent \aPerformance_c$
\item $\aPerformance_a \relBetter \aPerformance_b$, $\aPerformance_b \relEquivalent \aPerformance_c$, $\aPerformance_a \relBetter \aPerformance_c$
\item $\aPerformance_a \relWorse \aPerformance_b$, $\aPerformance_b \relIncomparable \aPerformance_c$, $\aPerformance_a \relIncomparable \aPerformance_c$
\item $\aPerformance_a \relWorse \aPerformance_b$, $\aPerformance_b \relIncomparable \aPerformance_c$, $\aPerformance_a \relWorse \aPerformance_c$
\item $\aPerformance_a \relWorse \aPerformance_b$, $\aPerformance_b \relBetter \aPerformance_c$, $\aPerformance_a \relIncomparable \aPerformance_c$
\item $\aPerformance_a \relWorse \aPerformance_b$, $\aPerformance_b \relBetter \aPerformance_c$, $\aPerformance_a \relBetter \aPerformance_c$
\item $\aPerformance_a \relWorse \aPerformance_b$, $\aPerformance_b \relBetter \aPerformance_c$, $\aPerformance_a \relWorse \aPerformance_c$
\item $\aPerformance_a \relWorse \aPerformance_b$, $\aPerformance_b \relBetter \aPerformance_c$, $\aPerformance_a \relEquivalent \aPerformance_c$
\item $\aPerformance_a \relWorse \aPerformance_b$, $\aPerformance_b \relWorse \aPerformance_c$, $\aPerformance_a \relWorse \aPerformance_c$
\item $\aPerformance_a \relWorse \aPerformance_b$, $\aPerformance_b \relEquivalent \aPerformance_c$, $\aPerformance_a \relWorse \aPerformance_c$
\item $\aPerformance_a \relEquivalent \aPerformance_b$, $\aPerformance_b \relIncomparable \aPerformance_c$, $\aPerformance_a \relIncomparable \aPerformance_c$
\item $\aPerformance_a \relEquivalent \aPerformance_b$, $\aPerformance_b \relBetter \aPerformance_c$, $\aPerformance_a \relBetter \aPerformance_c$
\item $\aPerformance_a \relEquivalent \aPerformance_b$, $\aPerformance_b \relWorse \aPerformance_c$, $\aPerformance_a \relWorse \aPerformance_c$
\item $\aPerformance_a \relEquivalent \aPerformance_b$, $\aPerformance_b \relEquivalent \aPerformance_c$, $\aPerformance_a \relEquivalent \aPerformance_c$
\end{enumerate}
\end{multicols}

From this list, we can derive some rules for manipulating the binary
relations $\relEquivalent$, $\relBetter$, $\relWorse$, and $\relIncomparable$.
First, we can see that $\relEquivalent$, $\relBetter$, and $\relWorse$
are transitive:

\begin{align}
\aPerformance_{1}\relEquivalent\aPerformance_{2}\wedge\aPerformance_{2}\relEquivalent\aPerformance_{3} & \Rightarrow\aPerformance_{1}\relEquivalent\aPerformance_{3}\label{eq:transitivity-rel-equivalent}\\
\aPerformance_{1}\relBetter\aPerformance_{2}\wedge\aPerformance_{2}\relBetter\aPerformance_{3} & \Rightarrow\aPerformance_{1}\relBetter\aPerformance_{3}\label{eq:transitivity-rel-better}\\
\aPerformance_{1}\relWorse\aPerformance_{2}\wedge\aPerformance_{2}\relWorse\aPerformance_{3} & \Rightarrow\aPerformance_{1}\relWorse\aPerformance_{3}\,.\label{eq:transitivity-rel-worse}
\end{align}
Second, we can also see how $\relEquivalent$ can be combined with
the other 3 relations:
\begin{align}
(\aPerformance_{1}\relEquivalent\aPerformance_{2}\wedge\aPerformance_{2}\relBetter\aPerformance_{3})\vee(\aPerformance_{1}\relBetter\aPerformance_{2}\wedge\aPerformance_{2}\relEquivalent\aPerformance_{3}) & \Rightarrow\aPerformance_{1}\relBetter\aPerformance_{3}\\
(\aPerformance_{1}\relEquivalent\aPerformance_{2}\wedge\aPerformance_{2}\relWorse\aPerformance_{3})\vee(\aPerformance_{1}\relWorse\aPerformance_{2}\wedge\aPerformance_{2}\relEquivalent\aPerformance_{3}) & \Rightarrow\aPerformance_{1}\relWorse\aPerformance_{3}\\
(\aPerformance_{1}\relEquivalent\aPerformance_{2}\wedge\aPerformance_{2}\relIncomparable\aPerformance_{3})\vee(\aPerformance_{1}\relIncomparable\aPerformance_{2}\wedge\aPerformance_{2}\relEquivalent\aPerformance_{3}) & \Rightarrow\aPerformance_{1}\relIncomparable\aPerformance_{3}\,.
\end{align}
And, third, we can see how $\relIncomparable$ can be combined with
$\relBetter$ and $\relWorse$:
\begin{align}
(\aPerformance_{1}\relIncomparable\aPerformance_{2}\wedge\aPerformance_{2}\relBetter\aPerformance_{3})\vee(\aPerformance_{1}\relBetter\aPerformance_{2}\wedge\aPerformance_{2}\relIncomparable\aPerformance_{3}) & \Rightarrow\aPerformance_{1}\relBetter\aPerformance_{3}\vee\aPerformance_{1}\relIncomparable\aPerformance_{3}\\
(\aPerformance_{1}\relIncomparable\aPerformance_{2}\wedge\aPerformance_{2}\relWorse\aPerformance_{3})\vee(\aPerformance_{1}\relWorse\aPerformance_{2}\wedge\aPerformance_{2}\relIncomparable\aPerformance_{3}) & \Rightarrow\aPerformance_{1}\relWorse\aPerformance_{3}\vee\aPerformance_{1}\relIncomparable\aPerformance_{3}\,.
\end{align}

\subsubsection{Properties of $\protect\relEquivalent$, $\protect\relBetter$, $\protect\relWorse$,
$\protect\relIncomparable$, $\protect\relWorseOrEquivalent$, and
$\protect\relBetterOrEquivalent$.}
\label{sec:Properties_of_orders}
\begin{lemma}
When $\relWorseOrEquivalent$ is a preorder, $\relEquivalent$ is
reflexive.
\end{lemma}

\begin{proof}
This results from the reflexivity of $\relWorseOrEquivalent$ and
from \cref{eq:rel-equivalent}: $\aPerformance\relEquivalent\aPerformance\Leftrightarrow\aPerformance\relWorseOrEquivalent\aPerformance\wedge\aPerformance\relWorseOrEquivalent\aPerformance\Leftrightarrow true$.
\end{proof}
\begin{lemma}
When $\relWorseOrEquivalent$ is a preorder, $\relEquivalent$ in
transitive.
\end{lemma}

\begin{proof}
This results from the transitivity of $\relWorseOrEquivalent$ (\cf \cref{eq:transitivity-rel-equivalent}).
\end{proof}
\begin{lemma}
When $\relWorseOrEquivalent$ is a preorder, $\relEquivalent$ is
symmetric.
\end{lemma}

\begin{proof}
This results from the fact that the conjunction is symmetric and from
\cref{eq:rel-equivalent}: $\aPerformance_{1}\relEquivalent\aPerformance_{2}\Leftrightarrow\aPerformance_{1}\relWorseOrEquivalent\aPerformance_{2}\wedge\aPerformance_{2}\relWorseOrEquivalent\aPerformance_{1}\Leftrightarrow\aPerformance_{2}\relWorseOrEquivalent\aPerformance_{1}\wedge\aPerformance_{1}\relWorseOrEquivalent\aPerformance_{2}\Leftrightarrow\aPerformance_{2}\relEquivalent\aPerformance_{1}$.
\end{proof}
\begin{lemma}
When $\relWorseOrEquivalent$ is a preorder, $\relBetter$ and $\relWorse$
are converse.
\end{lemma}

\begin{proof}
This results from the fact that the conjunction is symmetric and from
 \cref{eq:rel-better,eq:rel-worse}: $\aPerformance_{1}\relBetter\aPerformance_{2}\Leftrightarrow\aPerformance_{1}\not\relWorseOrEquivalent\aPerformance_{2}\wedge\aPerformance_{2}\relWorseOrEquivalent\aPerformance_{1}\Leftrightarrow\aPerformance_{2}\relWorseOrEquivalent\aPerformance_{1}\wedge\aPerformance_{1}\not\relWorseOrEquivalent\aPerformance_{2}\Leftrightarrow\aPerformance_{2}\relWorse\aPerformance_{1}$.
\end{proof}
\begin{lemma}
When $\relWorseOrEquivalent$ is a preorder, $\relBetter$ and $\relWorse$
are irreflexive.
\end{lemma}

\begin{proof}
For $\relBetter$, this results from the reflexivity of $\relWorseOrEquivalent$
and from \cref{eq:rel-better}: $\aPerformance\relBetter\aPerformance\Leftrightarrow\aPerformance\not\relWorseOrEquivalent\aPerformance\wedge\aPerformance\relWorseOrEquivalent\aPerformance\Leftrightarrow false\wedge true=false$.
For $\relWorse$, the proof is similar.
\end{proof}
\begin{lemma}
When $\relWorseOrEquivalent$ is a preorder, $\relBetter$ and $\relWorse$
are asymmetric.
\end{lemma}

\begin{proof}
For $\relBetter$, this is because $\aPerformance_{1}\relBetter\aPerformance_{2}\wedge\aPerformance_{2}\relBetter\aPerformance_{1}\Leftrightarrow(\aPerformance_{1}\not\relWorseOrEquivalent\aPerformance_{2}\wedge\aPerformance_{2}\relWorseOrEquivalent\aPerformance_{1})\wedge(\aPerformance_{2}\not\relWorseOrEquivalent\aPerformance_{1}\wedge\aPerformance_{1}\relWorseOrEquivalent\aPerformance_{2})\Leftrightarrow(\aPerformance_{1}\not\relWorseOrEquivalent\aPerformance_{2}\wedge\aPerformance_{1}\relWorseOrEquivalent\aPerformance_{2})\wedge(\aPerformance_{2}\not\relWorseOrEquivalent\aPerformance_{1}\wedge\aPerformance_{2}\relWorseOrEquivalent\aPerformance_{1})\Leftrightarrow false\wedge false\Leftrightarrow false$.
For $\relWorse$, the proof is similar.
\end{proof}
\begin{lemma}
When $\relWorseOrEquivalent$ is a preorder, $\relBetter$ and $\relWorse$
are transitive.
\end{lemma}

\begin{proof}
This results from the transitivity of $\relWorseOrEquivalent$ (\cf
 \cref{eq:transitivity-rel-better,eq:transitivity-rel-worse}).
\end{proof}
\begin{lemma}
When $\relWorseOrEquivalent$ is a preorder, $\relIncomparable$ is
irreflexive.
\end{lemma}

\begin{proof}
This results from the reflexivity of $\relWorseOrEquivalent$ and
from \cref{eq:rel-incomparable}: $\aPerformance\relIncomparable\aPerformance\Leftrightarrow\aPerformance\not\relWorseOrEquivalent\aPerformance\wedge\aPerformance\not\relWorseOrEquivalent\aPerformance\Leftrightarrow false\wedge false\Leftrightarrow false$.
\end{proof}
\begin{lemma}
When $\relWorseOrEquivalent$ is a preorder, $\relIncomparable$ is
symmetric.
\end{lemma}

\begin{proof}
This results from the fact that the conjunction is symmetric and from
\cref{eq:rel-incomparable}: $\aPerformance_{1}\relIncomparable\aPerformance_{2}\Leftrightarrow\aPerformance_{1}\not\relWorseOrEquivalent\aPerformance_{2}\wedge\aPerformance_{2}\not\relWorseOrEquivalent\aPerformance_{1}\Leftrightarrow\aPerformance_{2}\not\relWorseOrEquivalent\aPerformance_{1}\wedge\aPerformance_{1}\not\relWorseOrEquivalent\aPerformance_{2}\Leftrightarrow\aPerformance_{2}\relIncomparable\aPerformance_{1}$.
\end{proof}
\begin{lemma}
When $\relWorseOrEquivalent$ is a preorder, $\relWorseOrEquivalent$
and $\relBetterOrEquivalent$ are converse.
\end{lemma}

\begin{proof}
From \cref{eq:rel-better-or-equivalent,eq:rel-worse-or-equivalent},
as $\relBetter$ and $\relWorse$ are converse, we have $\aPerformance_{1}\relBetterOrEquivalent\aPerformance_{2}\Leftrightarrow\aPerformance_{1}\relBetter\aPerformance_{2}\vee\aPerformance_{1}\relEquivalent\aPerformance_{2}\Leftrightarrow\aPerformance_{2}\relWorse\aPerformance_{1}\vee\aPerformance_{2}\relEquivalent\aPerformance_{1}\Leftrightarrow\aPerformance_{2}\relWorseOrEquivalent\aPerformance_{1}$.
\end{proof}
\begin{lemma}
When $\relWorseOrEquivalent$ is a preorder, $\relWorseOrEquivalent$
and $\relBetterOrEquivalent$ are reflexive.
\end{lemma}

\begin{proof}
For $\relWorseOrEquivalent$, it is by definition of preorders. For
$\relBetterOrEquivalent$, from \cref{eq:rel-better-or-equivalent,eq:rel-better,eq:rel-equivalent}, we have $\aPerformance\relBetterOrEquivalent\aPerformance\Leftrightarrow\aPerformance\relBetter\aPerformance\vee\aPerformance\relEquivalent\aPerformance\Leftrightarrow(\aPerformance\not\relWorseOrEquivalent\aPerformance\wedge\aPerformance\relWorseOrEquivalent\aPerformance)\vee(\aPerformance\relWorseOrEquivalent\aPerformance\wedge\aPerformance\relWorseOrEquivalent\aPerformance)\Leftrightarrow(false\wedge true)\vee(true\wedge true)\Leftrightarrow true$.
\end{proof}
\begin{lemma}
When $\relWorseOrEquivalent$ is a preorder, $\relWorseOrEquivalent$
and $\relBetterOrEquivalent$ are transitive.
\end{lemma}

\begin{proof}
For $\relWorseOrEquivalent$, it is by definition of preorders. For
$\relBetterOrEquivalent$, as $\relWorseOrEquivalent$ and $\relBetterOrEquivalent$
are converse, $\aPerformance_{1}\relBetterOrEquivalent\aPerformance_{2}\wedge\aPerformance_{2}\relBetterOrEquivalent\aPerformance_{3}\Leftrightarrow\aPerformance_{3}\relWorseOrEquivalent\aPerformance_{2}\wedge\aPerformance_{2}\relWorseOrEquivalent\aPerformance_{1}\Rightarrow\aPerformance_{3}\relWorseOrEquivalent\aPerformance_{1}\Leftrightarrow\aPerformance_{1}\relBetterOrEquivalent\aPerformance_{3}$.
\end{proof}

\clearpage
\subsection{Supplementary Material about \cref{sec:mathematical-framework-importance}}
\label{sup:misleading_visual_inspection}

\subsubsection{The Visual Inspection of Formulas, to Determine the Importance given by Scores, can be Misleading!}

\global\long\def\weightA{\alpha}
\global\long\def\weightB{\beta}

Let us consider the example of two-class classification, with $\aPerformance(\eventTN)$, $\aPerformance(\eventFP)$, $\aPerformance(\eventFN)$, and $\aPerformance(\eventTP)$ denoting, respectively, the probability (or proportion) of true negatives, false positives, false negatives, and true positives. Here are two classical scores, the accuracy and the true positive rate:
\begin{equation*}
   \scoreAccuracy=\aPerformance(\eventTN)+\aPerformance(\eventTP)
   \qquad\qquad
   \scoreTPR=\frac{\aPerformance(\eventTP)}{\aPerformance(\eventFN)+\aPerformance(\eventTP)}
\end{equation*}
The formula for the accuracy gives the illusion that the same importance is given to $\eventTN$ and $\eventTP$ and that no importance at all is given to $\eventFP$ and $\eventFN$. For the true positive rate, the formula might give the impression that the same importance is given to $\eventFN$ and $\eventTP$ and that no importance at all is given to $\eventTN$ and $\eventFP$. In fact, the visual inspection of formulas like these is not reliable at all to judge the importance given by a score to the various events. To see it, consider rewriting the previous equations as
\begin{align*}
\scoreAccuracy & = 
   (1-\weightA)\aPerformance(\eventTN)
   -\weightA\aPerformance(\eventFP)
   -\weightA\aPerformance(\eventFN)
   +(1-\weightA)\aPerformance(\eventTP)
   +\weightA\qquad\forall\weightA\\
\scoreTPR & = \frac{
   -\weightA\aPerformance(\eventTN)
   -\weightA\aPerformance(\eventFP)
   -\weightA\aPerformance(\eventFN)
   +(1-\weightA)\aPerformance(\eventTP)
   +\weightA
}{
   -\weightB\aPerformance(\eventTN)
   -\weightB\aPerformance(\eventFP)
   +(1-\weightB)\aPerformance(\eventFN)
   +(1-\weightB)\aPerformance(\eventTP)
   +\weightB
}\qquad\forall\weightA,\weightB
\end{align*}
A visual inspection of such formulas would lead, indeed, to other illusions about the events that have no importance. This observation, however, should not stop us from thinking in terms of importance. In fact, we do it in this paper, but we do it in a mathematical framework that allows to do it rigorously.

\subsubsection{So, How can we Determine the Importance given by Scores?}

One cannot determine the application-specific preferences implicitly considered by any score $\aScore$. However, one can determine those that are consistent, or have the best consistency, with $\aScore$. We detail these notions hereafter.

\paragraph{For a given set of performances.}
Consider a score $\aScore$ and a ranking score $\rankingScore$. If, for a given set $\aSetOfPerformances\subseteq\allPerformances$, the scores $\aScore$ and $\rankingScore$ are linked by a strict monotonic relationship on $\aSetOfPerformances \cap \domainOfScore[\aScore] \cap \domainOfScore[\rankingScore]$ and if $\aSetOfPerformances \cap \domainOfScore[\aScore] = \aSetOfPerformances \cap \domainOfScore[\rankingScore]$, then the performance orderings $\ordering_{\aScore}$ and $\ordering_{\rankingScore}$ (induced by $\aScore$ and $\rankingScore$ in the way specified in \cref{thm:suff-cond-preorder}) are identical on $\aSetOfPerformances$. We say that the score $\aScore$ \emph{is consistent with} the application-specific preferences $\randVarImportance$, on this set. A score can be consistent with different importance values (\eg, as consequence of Properties~\ref{prop:scale-invariance} and~\ref{prop:scale-invariance-per-satisfaction}).

\long\def\interior{\mathring{\mathbb{P}}}

\paragraph{For a given distribution of performances.}
Let $\interior=\{\aPerformance\in\allPerformances:\aPerformance(\{\aSample\})>0 \forall \aSample\in\sampleSpace)\}$. All ranking scores are defined on this set. Consider a score $\aScore$ and the set $\aSetOfPerformances=\domainOfScore[\aScore]\cap\interior$. We say that the score $\aScore$ \emph{has the best consistency with} the application-specific preferences $\randVarImportance$ when $\randVarImportance$ maximizes the rank correlation between $\aScore$ and $\rankingScore$, on $\aSetOfPerformances$, for the given distribution of performances. See \cref{sec:kendall} for computational details.

\clearpage
\subsection{Supplementary Material about \cref{sec:sufficient-conditions}}
\label{sec:sufficient-conditions-proofs}

\subsubsection{Proof of Theorem~\ref{thm:suff-cond-preorder}.}

For convenience, we provide a reminder of Theorem~\ref{thm:suff-cond-preorder} and Axiom~\ref{axiom:preorder} below.
\restatableTA*
\restatableAA*

\begin{proof}
To establish that $\ordering_{\aScore}$ is a preorder, we
have to show that it is (1)~reflexive and (2)~transitive.
\begin{enumerate}
\item[(1)] The reflexivity of $\ordering_{\aScore}$ is trivial to establish,
since $\aPerformance_{1}=\aPerformance_{2}\Rightarrow\aPerformance_{1}\ordering_{\aScore}\aPerformance_{2}$.
\item[(2)] The transitivity of $\ordering_{\aScore}$ can be shown as follows.
$\aPerformance_{1}\ordering_{\aScore}\aPerformance_{2}\wedge\aPerformance_{2}\ordering_{\aScore}\aPerformance_{3}$
implies that:
\begin{itemize}
\item either $\aPerformance_{1}\in\domainOfScore$, $\aPerformance_{2}\in\domainOfScore$,
$\aPerformance_{3}\in\domainOfScore$, and $\aScore(\aPerformance_{1})\le\aScore(\aPerformance_{2})\wedge\aScore(\aPerformance_{2})\le\aScore(\aPerformance_{3})\Rightarrow\aScore(\aPerformance_{1})\le\aScore(\aPerformance_{3})\Rightarrow\aPerformance_{1}\ordering_{\aScore}\aPerformance_{3}$;
\item or $\aPerformance_{1}\not\in\domainOfScore$, $\aPerformance_{2}\not\in\domainOfScore$,
$\aPerformance_{3}\not\in\domainOfScore$, and $\aPerformance_{1}=\aPerformance_{2}\wedge\aPerformance_{2}=\aPerformance_{3}\Rightarrow\aPerformance_{1}=\aPerformance_{3}\Rightarrow\aPerformance_{1}\ordering_{\aScore}\aPerformance_{3}$.
\end{itemize}
\end{enumerate}
We conclude that, in all cases, $\aPerformance\ordering_{\aScore}\aPerformance$
and $\aPerformance_{1}\ordering_{\aScore}\aPerformance_{2}\wedge\aPerformance_{2}\ordering_{\aScore}\aPerformance_{3}\Rightarrow\aPerformance_{1}\ordering_{\aScore}\aPerformance_{3}$.
The orderings $\ordering_{\aScore}$ induced by scores $\aScore$
are thus preorders.
\end{proof}

\paragraph*{Summary.}

If the homogeneous binary relations $\relEquivalent$, $\relBetter$,
$\relWorse$, and $\relIncomparable$ on $\allPerformances$ are derived
from the ordering $\ordering_{\aScore}$ as explained above, and if
the $\ordering_{\aScore}$ is derived from the score $\aScore$, then
the comparison between performances $\aPerformance_{1}$ and $\aPerformance_{2}$
can be summarized as follows.
\begin{center}
\begin{tabular}{c|c|c|}
 & $\aPerformance_{1}\in\domainOfScore$ & $\aPerformance_{1}\not\in\domainOfScore$\tabularnewline
\hline 
$\aPerformance_{2}\in\domainOfScore$ & {$\!
\begin{aligned}
\aScore(\aPerformance_{1})<\aScore(\aPerformance_{2}) & \Leftrightarrow\aPerformance_{1}\relWorse\aPerformance_{2}\\
\aScore(\aPerformance_{1})=\aScore(\aPerformance_{2}) & \Leftrightarrow\aPerformance_{1}\relEquivalent\aPerformance_{2}\\
\aScore(\aPerformance_{1})>\aScore(\aPerformance_{2}) & \Leftrightarrow\aPerformance_{1}\relBetter\aPerformance_{2}
\end{aligned}$
} & {$\aPerformance_{1}\relIncomparable\aPerformance_{2}$}\tabularnewline
\hline 
$\aPerformance_{2}\not\in\domainOfScore$ & $\aPerformance_{1}\relIncomparable\aPerformance_{2}$ & {$\!
\begin{aligned}
\aPerformance_{1}=\aPerformance_{2} & \Leftrightarrow\aPerformance_{1}\relEquivalent\aPerformance_{2}\\
\aPerformance_{1}\ne\aPerformance_{2} & \Leftrightarrow\aPerformance_{1}\relIncomparable\aPerformance_{2}
\end{aligned}
$}\tabularnewline
\hline 
\end{tabular}
\par\end{center}

\subsubsection{Proof of Theorem~\ref{thm:suff-cond-satisfaction}}

For convenience, we provide a reminder of Theorem~\ref{thm:suff-cond-satisfaction} and Axiom~\ref{axiom:satisfaction} below.
\restatableTBC*
\restatableAB*

\begin{proof}
Axiom~\ref{axiom:satisfaction} is satisfied when $\aPerformance_{1}\not\in\domainOfScore$
or $\aPerformance_{2}\not\in\domainOfScore$.
\begin{itemize}
\item Either $\aPerformance_{1}=\aPerformance_{2}\Leftrightarrow\aPerformance_{1}\relEquivalent\aPerformance_{2}\Rightarrow\aPerformance_{1}\relWorseOrEquivalent\aPerformance_{2}$,
\item or $\aPerformance_{1}\ne\aPerformance_{2}\Leftrightarrow\aPerformance_{1}\relIncomparable\aPerformance_{2}$.
\end{itemize}
Axiom~\ref{axiom:satisfaction} is also satisfied when
$\aPerformance_{1}\in\domainOfScore$ and $\aPerformance_{2}\in\domainOfScore$.
\begin{itemize}
\item On the one hand, the axiom stipulates that the event $\anEvent_{1}=\{\aSample\in\sampleSpace:\randVarSatisfaction(\aSample)\le s\}$
and the performance $\aPerformance_{1}$ are such that $\aPerformance_{1}(\anEvent_{1})=1$.
Trivially, we have $\max_{\aSample\in\anEvent_{1}}\randVarSatisfaction(\aSample)\le s$.
On the other hand, the theorem stipulates that, as $\aPerformance_{1}(\anEvent_{1})=1$,
$\aScore(\aPerformance_{1})\le\max_{\aSample\in\anEvent_{1}}\randVarSatisfaction(\aSample)$.
Putting all together, we have $\aScore(\aPerformance_{1})\le s$.
\item On the one hand, the axiom stipulates that the event $\anEvent_{2}=\{\aSample\in\sampleSpace:\randVarSatisfaction(\aSample)\ge s\}$
and the performance $\aPerformance_{2}$ are such that $\aPerformance_{2}(\anEvent_{2})=1$.
Trivially, we have $s\le\min_{\aSample\in\anEvent_{2}}\randVarSatisfaction(\aSample)$.
On the other hand, the theorem stipulates that, as $\aPerformance_{2}(\anEvent_{2})=1$,
$\min_{\aSample\in\anEvent_{2}}\randVarSatisfaction(\aSample)\le\aScore(\aPerformance_{2})$.
Putting all together, we have $s\le\aScore(\aPerformance_{2})$.
\item As we have established that $\aScore(\aPerformance_{1})\le s$ and
$s\le\aScore(\aPerformance_{2})$, we have $\aScore(\aPerformance_{1})\le\aScore(\aPerformance_{2})\Leftrightarrow\aPerformance_{1}\relWorseOrEquivalent\aPerformance_{2}$.
\end{itemize}
\end{proof}

\subsubsection{Proof of Theorem~\ref{thm:suff-cond-combinations}}

For convenience, we provide a reminder of Theorem~\ref{thm:suff-cond-combinations} and Axiom~\ref{axiom:combinations} below.
\restatableTD*
\restatableAD*

\begin{proof}

    We take $\ordering=\ordering_{\aScore}$ and $\aSetOfPerformances\ne\emptyset$.

    \mysection{Remainder of the conditions.}
    The first condition of \cref{thm:suff-cond-combinations} is:
    \begin{equation}
        \aSetOfPerformances\subseteq\domainOfScore\Rightarrow\achievableByCombinations(\aSetOfPerformances)\subseteq\domainOfScore\point \label{eq:thm-suff-cond-combinations-condition-1}
    \end{equation}
    The second condition of \cref{thm:suff-cond-combinations} is:
    \begin{equation}
        \min_{\aPerformance\in\aSetOfPerformances}\aScore(\aPerformance) \le \aScore(\overline{\aPerformance}) \le \max_{\aPerformance\in\aSetOfPerformances} \aScore(\aPerformance) \qquad \forall \, \aSetOfPerformances\subseteq\domainOfScore \qquad \forall \, \overline{\aPerformance}\in\achievableByCombinations(\aSetOfPerformances)
        \point \label{eq:thm-suff-cond-combinations-condition-2}
    \end{equation}
    The condition of Axiom \ref{axiom:combinations} is that $\aPerformance$ is comparable to all performances in the set $\aSetOfPerformances$:
    \begin{equation}
        \aPerformance'\relWorseOrEquivalent\aPerformance\vee\aPerformance\relWorseOrEquivalent\aPerformance' \qquad \forall \, \aPerformance'\in\aSetOfPerformances \point \label{eq:axiom-combinations-condition}
    \end{equation}

    \mysection{On the domain of $\aScore$.}
    By \cref{thm:suff-cond-preorder}, this last condition implies that
    \begin{equation}
        \aPerformance\in\domainOfScore \comma \label{eq:P-in-dom}
    \end{equation}
    and
    \begin{equation}
        \aSetOfPerformances\subseteq\domainOfScore \point \label{eq:Pi-subset-dom}
    \end{equation}
    Taking \cref{eq:thm-suff-cond-combinations-condition-1} and \cref{eq:Pi-subset-dom} together, we have
    \begin{equation}
        \achievableByCombinations(\aSetOfPerformances)\subseteq\domainOfScore 
        \qquad \Leftrightarrow \qquad
        \overline{\aPerformance}\in\domainOfScore\qquad \forall\,\overline{\aPerformance}\in\achievableByCombinations(\aSetOfPerformances) \point
    \end{equation}

    \mysection{Proof that $\aPerformance'\relWorseOrEquivalent\aPerformance\, \forall\aPerformance'\in\aSetOfPerformances \Rightarrow \overline{\aPerformance}\relWorseOrEquivalent\aPerformance\, \forall\overline{\aPerformance}\in\achievableByCombinations(\aSetOfPerformances)$.}
    On the one hand, we have, by \cref{thm:suff-cond-preorder},
    \begin{align*}
        \aPerformance'\relWorseOrEquivalent\aPerformance\, \forall\aPerformance'\in\aSetOfPerformances
        \Leftrightarrow & 
        \aScore(\aPerformance')\le \aScore(\aPerformance) \qquad \forall\,\aPerformance'\in\aSetOfPerformances \\
        \Leftrightarrow & 
        \max_{\aPerformance'\in\aSetOfPerformances} \aScore(\aPerformance') \le \aScore(\aPerformance) \point
    \end{align*}
    On the other hand, \cref{eq:thm-suff-cond-combinations-condition-2} implies that
    \begin{equation*}
        \aScore(\overline{\aPerformance}) \le \max_{\aPerformance'\in\aSetOfPerformances} \aScore(\aPerformance') \qquad \forall \, \overline{\aPerformance}\in\achievableByCombinations(\aSetOfPerformances)
        \point
    \end{equation*}
    Considering the last two equations together, we obtain
    \begin{align*}
        & \aScore(\overline{\aPerformance}) \le \max_{\aPerformance'\in\aSetOfPerformances} \aScore(\aPerformance')  \le  \aScore(\aPerformance) \qquad \forall \, \overline{\aPerformance}\in\achievableByCombinations(\aSetOfPerformances) \\
        \Rightarrow & \aScore(\overline{\aPerformance}) \le \aScore(\aPerformance) \qquad \forall \, \overline{\aPerformance}\in\achievableByCombinations(\aSetOfPerformances) \point
    \end{align*}
    By \cref{thm:suff-cond-preorder}, we have thus $\overline{\aPerformance}\relWorseOrEquivalent\aPerformance$.

    \mysection{Proof that $\aPerformance'\not\relWorseOrEquivalent\aPerformance\, \forall\aPerformance'\in\aSetOfPerformances \Rightarrow \overline{\aPerformance}\not\relWorseOrEquivalent\aPerformance\, \forall\overline{\aPerformance}\in\achievableByCombinations(\aSetOfPerformances)$.} 
    On the one hand, we have, by \cref{eq:axiom-combinations-condition} and \cref{thm:suff-cond-preorder},
    \begin{align*}
        \aPerformance'\not\relWorseOrEquivalent\aPerformance\, \forall\,\aPerformance'\in\aSetOfPerformances
        \Rightarrow & \aPerformance \relWorse \aPerformance' \, \forall\,\aPerformance'\in\aSetOfPerformances\\
        \Leftrightarrow & \aScore(\aPerformance) < \aScore(\aPerformance') \qquad \forall\,\aPerformance'\in\aSetOfPerformances\\
        \Leftrightarrow & \aScore(\aPerformance) < \min_{\aPerformance'\in\aSetOfPerformances}\aScore(\aPerformance')
    \end{align*}
    On the other hand, \cref{eq:thm-suff-cond-combinations-condition-2} implies that
    \begin{equation*}
        \min_{\aPerformance'\in\aSetOfPerformances}\aScore(\aPerformance') \le \aScore(\overline{\aPerformance}) \qquad \forall \, \overline{\aPerformance}\in\achievableByCombinations(\aSetOfPerformances)
    \end{equation*}
    Considering the last two equations together, we obtain
    \begin{align*}
        & \aScore(\aPerformance) < \min_{\aPerformance'\in\aSetOfPerformances}\aScore(\aPerformance') \le \aScore(\overline{\aPerformance}) \qquad \forall \, \overline{\aPerformance}\in\achievableByCombinations(\aSetOfPerformances) \\
        \Rightarrow & \aScore(\aPerformance) < \aScore(\overline{\aPerformance}) \qquad \forall \, \overline{\aPerformance}\in\achievableByCombinations(\aSetOfPerformances)  \point
    \end{align*}
    By \cref{thm:suff-cond-preorder}, we have thus $\aPerformance \relWorse \overline{\aPerformance}$, and by \cref{eq:axiom-combinations-condition}, $\overline{\aPerformance}\not\relWorseOrEquivalent\aPerformance$.

    \mysection{Proof that $\aPerformance\relWorseOrEquivalent\aPerformance'\, \forall\aPerformance'\in\aSetOfPerformances \Rightarrow \aPerformance\relWorseOrEquivalent\overline{\aPerformance}\, \forall\overline{\aPerformance}\in\achievableByCombinations(\aSetOfPerformances)$.} 
    On the one hand, we have, by \cref{thm:suff-cond-preorder},
    \begin{align*}
        \aPerformance\relWorseOrEquivalent\aPerformance'\, \forall\aPerformance'\in\aSetOfPerformances
        \Leftrightarrow & 
        \aScore(\aPerformance)\le \aScore(\aPerformance') \qquad \forall\,\aPerformance'\in\aSetOfPerformances \\
        \Leftrightarrow & 
        \aScore(\aPerformance) \le \min_{\aPerformance'\in\aSetOfPerformances} \aScore(\aPerformance') \point
    \end{align*}
    On the other hand, \cref{eq:thm-suff-cond-combinations-condition-2} implies that
    \begin{equation*}
        \min_{\aPerformance'\in\aSetOfPerformances} \aScore(\aPerformance') \le  \aScore(\overline{\aPerformance}) \qquad \forall \, \overline{\aPerformance}\in\achievableByCombinations(\aSetOfPerformances)
        \point
    \end{equation*}
    Considering the last two equations together, we obtain
    \begin{align*}
        & \aScore(\aPerformance) \le \min_{\aPerformance'\in\aSetOfPerformances} \aScore(\aPerformance')  \le  \aScore(\overline{\aPerformance}) \qquad \forall \, \overline{\aPerformance}\in\achievableByCombinations(\aSetOfPerformances) \\
        \Rightarrow & \aScore(\aPerformance) \le \aScore(\overline{\aPerformance}) \qquad \forall \, \overline{\aPerformance}\in\achievableByCombinations(\aSetOfPerformances) \point
    \end{align*}
    By \cref{thm:suff-cond-preorder}, we have thus $\aPerformance\relWorseOrEquivalent\overline{\aPerformance}$.

    \mysection{Proof that $\aPerformance\not\relWorseOrEquivalent\aPerformance'\, \forall\aPerformance'\in\aSetOfPerformances \Rightarrow \aPerformance\not\relWorseOrEquivalent\overline{\aPerformance}\, \forall\overline{\aPerformance}\in\achievableByCombinations(\aSetOfPerformances)$.} 
    On the one hand, we have, by \cref{eq:axiom-combinations-condition} and \cref{thm:suff-cond-preorder},
    \begin{align*}
        \aPerformance\not\relWorseOrEquivalent\aPerformance'\, \forall\,\aPerformance'\in\aSetOfPerformances
        \Rightarrow & \aPerformance' \relWorse \aPerformance \, \forall\,\aPerformance'\in\aSetOfPerformances\\
        \Leftrightarrow & \aScore(\aPerformance') < \aScore(\aPerformance) \qquad \forall\,\aPerformance'\in\aSetOfPerformances\\
        \Leftrightarrow &  \max_{\aPerformance'\in\aSetOfPerformances}\aScore(\aPerformance') < \aScore(\aPerformance)
    \end{align*}
    On the other hand, \cref{eq:thm-suff-cond-combinations-condition-2} implies that
    \begin{equation*}
        \aScore(\overline{\aPerformance}) \le \max_{\aPerformance'\in\aSetOfPerformances}\aScore(\aPerformance') \qquad \forall \, \overline{\aPerformance}\in\achievableByCombinations(\aSetOfPerformances)
    \end{equation*}
    Considering the last two equations together, we obtain
    \begin{align*}
        & \aScore(\overline{\aPerformance}) \le \max_{\aPerformance'\in\aSetOfPerformances}\aScore(\aPerformance') < \aScore(\aPerformance) \qquad \forall \, \overline{\aPerformance}\in\achievableByCombinations(\aSetOfPerformances) \\
        \Rightarrow & \aScore(\overline{\aPerformance}) < \aScore(\aPerformance) \qquad \forall \, \overline{\aPerformance}\in\achievableByCombinations(\aSetOfPerformances)  \point
    \end{align*}
    By \cref{thm:suff-cond-preorder}, we have thus $\overline{\aPerformance} \relWorse \aPerformance$, and by \cref{eq:axiom-combinations-condition}, $\aPerformance\not\relWorseOrEquivalent\overline{\aPerformance}$.
    
\end{proof}

\clearpage
\subsection{Supplementary Material about \cref{sec:ranking-scores}}

\subsubsection{All Ranking Scores can be Used to Rank Performances (for \texorpdfstring{$\achievableByCombinations=\allConvexCombinations$}{phi=conv})}
\label{sec:ranking-scores-proofs}

To show that all ranking scores can be used to rank performances, for $\achievableByCombinations=\allConvexCombinations$, we show that these scores satisfy the conditions of Theorems~\ref{thm:suff-cond-preorder}, \ref{thm:suff-cond-satisfaction}, and~\ref{thm:suff-cond-combinations}.


\paragraph{All ranking scores satisfy the conditions of Theorem~\ref{thm:suff-cond-preorder}, and thus Axiom~\ref{axiom:preorder}.}

For convenience, we provide a reminder of Theorem~\ref{thm:suff-cond-preorder} and Axiom~\ref{axiom:preorder} below.
\restatableTA*
\restatableAA*

\begin{theorem}
All ranking scores satisfy the conditions of Theorem~\ref{thm:suff-cond-preorder}.
\end{theorem}

\begin{proof}
For all ranking scores $\rankingScore$, it is possible to induce an ordering $\ordering_{\rankingScore}$ satisfying the requirements of Theorem~\ref{thm:suff-cond-preorder}.
\end{proof}


\paragraph{All ranking scores satisfy the conditions of Theorem~\ref{thm:suff-cond-satisfaction}, and thus Axiom~\ref{axiom:satisfaction}.}

For convenience, we provide a reminder of Theorem~\ref{thm:suff-cond-satisfaction} and Axiom~\ref{axiom:satisfaction} below.
\restatableTBC*
\restatableAB*

\begin{theorem}
All ranking scores satisfy the conditions of Theorem~\ref{thm:suff-cond-satisfaction}.
\end{theorem}

\begin{proof}
We take $\aScore=\rankingScore$. When $\aPerformance(\anEvent)=1$,
we have
\[
\rankingScore(\aPerformance)=\frac{\sum_{\aSample\in\sampleSpace}\randVarImportance(\aSample)\randVarSatisfaction(\aSample)\aPerformance(\{\aSample\})}{\sum_{\aSample\in\sampleSpace}\randVarImportance(\aSample)\aPerformance(\{\aSample\})}=\frac{\sum_{\aSample\in\anEvent}\randVarImportance(\aSample)\randVarSatisfaction(\aSample)\aPerformance(\{\aSample\})}{\sum_{\aSample\in\anEvent}\randVarImportance(\aSample)\aPerformance(\{\aSample\})}
\]
with $\sum_{\aSample\in\sampleSpace}\randVarImportance(\aSample)\aPerformance(\{\aSample\})>0$
when $\aPerformance\in\domainOfScore[{\rankingScore}]$.
\begin{itemize}
\item Let $M=\max_{\aSample\in\anEvent}\randVarSatisfaction(\aSample)$.
We have
\begin{align*}
 & \randVarSatisfaction(\aSample)\le M\qquad\forall\aSample\in\anEvent\\
\Leftrightarrow & \randVarSatisfaction(\aSample)-M\le0\qquad\forall\aSample\in\anEvent\\
\Rightarrow & \sum_{\aSample\in\anEvent}\randVarImportance(\aSample)\left[\randVarSatisfaction(\aSample)-M\right]\aPerformance(\{\aSample\})\le0\qquad\textrm{ as }\randVarImportance(\aSample)\ge0\textrm{ and }\aPerformance(\{\aSample\})\ge0\\
\Leftrightarrow & \sum_{\aSample\in\anEvent}\randVarImportance(\aSample)\randVarSatisfaction(\aSample)\aPerformance(\{\aSample\})\le\sum_{\aSample\in\anEvent}\randVarImportance(\aSample)M\aPerformance(\{\aSample\})\\
\Leftrightarrow & \sum_{\aSample\in\anEvent}\randVarImportance(\aSample)\randVarSatisfaction(\aSample)\aPerformance(\{\aSample\})\le M\underbrace{\sum_{\aSample\in\anEvent}\randVarImportance(\aSample)\aPerformance(\{\aSample\})}_{>0}\\
\Leftrightarrow & \frac{\sum_{\aSample\in\anEvent}\randVarImportance(\aSample)\randVarSatisfaction(\aSample)\aPerformance(\{\aSample\})}{\sum_{\aSample\in\anEvent}\randVarImportance(\aSample)\aPerformance(\{\aSample\})}\le M\\
\Leftrightarrow & \rankingScore(\aPerformance)\le M
\end{align*}
\item Let $m=\min_{\aSample\in\anEvent}\randVarSatisfaction(\aSample)$.
We have
\begin{align*}
 & \randVarSatisfaction(\aSample)\ge m\qquad\forall\aSample\in\anEvent\\
\Leftrightarrow & \randVarSatisfaction(\aSample)-m\ge0\qquad\forall\aSample\in\anEvent\\
\Rightarrow & \sum_{\aSample\in\anEvent}\randVarImportance(\aSample)\left[\randVarSatisfaction(\aSample)-m\right]\aPerformance(\{\aSample\})\ge0\qquad\textrm{ as }\randVarImportance(\aSample)\ge0\textrm{ and }\aPerformance(\{\aSample\})\ge0\\
\Leftrightarrow & \sum_{\aSample\in\anEvent}\randVarImportance(\aSample)\randVarSatisfaction(\aSample)\aPerformance(\{\aSample\})\ge\sum_{\aSample\in\anEvent}\randVarImportance(\aSample)m\aPerformance(\{\aSample\})\\
\Leftrightarrow & \sum_{\aSample\in\anEvent}\randVarImportance(\aSample)\randVarSatisfaction(\aSample)\aPerformance(\{\aSample\})\ge m\underbrace{\sum_{\aSample\in\anEvent}\randVarImportance(\aSample)\aPerformance(\{\aSample\})}_{>0}\\
\Leftrightarrow & \frac{\sum_{\aSample\in\anEvent}\randVarImportance(\aSample)\randVarSatisfaction(\aSample)\aPerformance(\{\aSample\})}{\sum_{\aSample\in\anEvent}\randVarImportance(\aSample)\aPerformance(\{\aSample\})}\ge m\\
\Leftrightarrow & \rankingScore(\aPerformance)\ge m
\end{align*}
\item Putting all together, when $\aPerformance(\anEvent)=1$, we have $m\le\rankingScore(\aPerformance)\le M$, and so,  
\begin{equation}
    \min_{\aSample\in\anEvent}\randVarSatisfaction(\aSample)\le\rankingScore(\aPerformance)\le\max_{\aSample\in\anEvent}\randVarSatisfaction(\aSample)\,.
\end{equation}
\end{itemize}
\end{proof}


\paragraph{All ranking scores satisfy the conditions of Theorem~\ref{thm:suff-cond-combinations}, and thus Axiom~\ref{axiom:combinations} (for $\achievableByCombinations=\allConvexCombinations$).}

For convenience, we provide a reminder of Theorem~\ref{thm:suff-cond-combinations} and Axiom~\ref{axiom:combinations} below.
\restatableTD*
\restatableAD*

\begin{theorem}
All ranking scores satisfy the conditions of Theorem~\ref{thm:suff-cond-combinations} (for $\achievableByCombinations=\allConvexCombinations$).
\end{theorem}

\begin{proof}
The proof is in two parts.
\begin{itemize}
\item First, let us show that $\aSetOfPerformances\subseteq\domainOfScore[{\rankingScore}]\Rightarrow\allConvexCombinations(\aSetOfPerformances)\subseteq\domainOfScore[{\rankingScore}]$.
For any $\overline{\aPerformance}\in\allConvexCombinations(\aSetOfPerformances)$
there exists a weighting function $\lambda_{\aSetOfPerformances,\overline{\aPerformance}}:\aSetOfPerformances\rightarrow\realNumbers_{\ge0}:\aPerformance\mapsto\lambda_{\aSetOfPerformances,\overline{\aPerformance}}(\aPerformance)$
such that $\sum_{\aPerformance\in\aSetOfPerformances}\lambda_{\aSetOfPerformances,\overline{\aPerformance}}(\aPerformance)=1$
and $\sum_{\aPerformance\in\aSetOfPerformances}\lambda_{\aSetOfPerformances,\overline{\aPerformance}}(\aPerformance)\aPerformance=\overline{\aPerformance}$.
For all $\overline{\aPerformance}\in\allConvexCombinations(\aSetOfPerformances)$,
we have:
\begin{align*}
\aSetOfPerformances\subseteq\domainOfScore[{\rankingScore}] & \Leftrightarrow\sum_{\aSample\in\sampleSpace}\randVarImportance(\aSample)\aPerformance(\{\aSample\})\ne0\qquad\forall\aPerformance\in\aSetOfPerformances\\
 & \Rightarrow\sum_{\aPerformance\in\aSetOfPerformances}\lambda_{\aSetOfPerformances,\overline{\aPerformance}}(\aPerformance)\sum_{\aSample\in\sampleSpace}\randVarImportance(\aSample)\aPerformance(\{\aSample\})\ne0\\
 & \Leftrightarrow\sum_{\aSample\in\sampleSpace}\randVarImportance(\aSample)\sum_{\aPerformance\in\aSetOfPerformances}\lambda_{\aSetOfPerformances,\overline{\aPerformance}}\aPerformance(\{\aSample\})\ne0\\
 & \Leftrightarrow\sum_{\aSample\in\sampleSpace}\randVarImportance(\aSample)\overline{\aPerformance}(\{\aSample\})\ne0\\
 & \Leftrightarrow\overline{\aPerformance}\in\domainOfScore[{\rankingScore}]\,.
\end{align*}
\item Second, let us show that, for all $\overline{\aPerformance}\in\allConvexCombinations(\aSetOfPerformances)$,
$\min_{\aPerformance\in\aSetOfPerformances}\rankingScore(\aPerformance)\le\rankingScore(\overline{\aPerformance})\le\max_{\aPerformance\in\aSetOfPerformances}\rankingScore(\aPerformance)$.
Let us pose $l=\min_{\aPerformance\in\aSetOfPerformances}\rankingScore(\aPerformance)$
and $u=\max_{\aPerformance\in\aSetOfPerformances}\rankingScore(\aPerformance)$.
We have:
\begin{align*}
 & l\le\rankingScore(\aPerformance)\le u\qquad\forall\aPerformance\in\aSetOfPerformances\\
\Leftrightarrow & l\le\frac{\sum_{\aSample\in\sampleSpace}\randVarImportance(\aSample)\randVarSatisfaction(\aSample)\aPerformance(\{\aSample\})}{\sum_{\aSample\in\sampleSpace}\randVarImportance(\aSample)\aPerformance(\{\aSample\})}\le u\qquad\forall\aPerformance\in\aSetOfPerformances\\
\Leftrightarrow & l\sum_{\aSample\in\sampleSpace}\randVarImportance(\aSample)\aPerformance(\{\aSample\})\le\sum_{\aSample\in\sampleSpace}\randVarImportance(\aSample)\randVarSatisfaction(\aSample)\aPerformance(\{\aSample\})\le u\sum_{\aSample\in\sampleSpace}\randVarImportance(\aSample)\aPerformance(\{\aSample\})\qquad\forall\aPerformance\in\aSetOfPerformances\\
\Rightarrow & l\sum_{\aSample\in\sampleSpace}\randVarImportance(\aSample)\overline{\aPerformance}(\{\aSample\})\le\sum_{\aSample\in\sampleSpace}\randVarImportance(\aSample)\randVarSatisfaction(\aSample)\overline{\aPerformance}(\{\aSample\})\le u\sum_{\aSample\in\sampleSpace}\randVarImportance(\aSample)\overline{\aPerformance}(\{\aSample\})\\
\Leftrightarrow & l\le\frac{\sum_{\aSample\in\sampleSpace}\randVarImportance(\aSample)\randVarSatisfaction(\aSample)\overline{\aPerformance}(\{\aSample\})}{\sum_{\aSample\in\sampleSpace}\randVarImportance(\aSample)\overline{\aPerformance}(\{\aSample\})}\le u\\
\Leftrightarrow & l\le\rankingScore(\overline{\aPerformance})\le u
\end{align*}
\end{itemize}
\end{proof}

\subsubsection{On the Properties of Ranking Scores.}

\begin{proof}[Proof of Property~\ref{prop:decomposition-importance-satisfaction}]
    Let us demonstrate that we have $\rankingScore(\aPerformance)=\scoreExpectedSatisfaction(\aPerformance')$ with $\aPerformance'=\opFilter(\aPerformance)$. For all $\aSample\in\sampleSpace$, we have:
    \begin{equation*}
        \aPerformance'(\{\aSample\})=\frac{
            \aPerformance(\{\aSample\}) \randVarImportance(\aSample)
        }{
            \sum_{\aSample'\in\sampleSpace} \aPerformance(\{\aSample'\}) \randVarImportance(\aSample')
        }
    \end{equation*}
    Thus,
    \begin{align*}
        \scoreExpectedSatisfaction(\aPerformance') & = \sum_{\aSample\in\sampleSpace} \aPerformance'(\{\aSample\}) \randVarSatisfaction(\aSample) \\
        & = \sum_{\aSample\in\sampleSpace} \frac{
            \aPerformance(\{\aSample\}) \randVarImportance(\aSample)
        }{
            \sum_{\aSample'\in\sampleSpace} \aPerformance(\{\aSample'\}) \randVarImportance(\aSample')
        } \randVarSatisfaction(\aSample) \\
        & = \frac{
            \sum_{\aSample\in\sampleSpace} \aPerformance(\{\aSample\}) \randVarImportance(\aSample) \randVarSatisfaction(\aSample)
        }{
            \sum_{\aSample'\in\sampleSpace} \aPerformance(\{\aSample'\}) \randVarImportance(\aSample')
        } \\
        & = \rankingScore(\aPerformance)
    \end{align*}
\end{proof}

\begin{proof}[Proof of Property~\ref{prop:linear-transformation-satisfaction}]
Let $\randVarSatisfaction'=\alpha\randVarSatisfaction+\beta$ with $\alpha,\beta\in\realNumbers$.
\begin{equation*}
    \frac{
        \sum_{\aSample\in\sampleSpace} \aPerformance(\{\aSample\}) \randVarSatisfaction'(\aSample) \randVarImportance (\aSample)
    }{
        \sum_{\aSample\in\sampleSpace} \aPerformance(\{\aSample\}) \randVarImportance (\aSample)
    }
    =
    \alpha
    \frac{
        \sum_{\aSample\in\sampleSpace} \aPerformance(\{\aSample\}) \randVarSatisfaction(\aSample) \randVarImportance (\aSample)
    }{
        \sum_{\aSample\in\sampleSpace} \aPerformance(\{\aSample\}) \randVarImportance (\aSample)
    }
    + \beta
\end{equation*}
\end{proof}

\begin{proof}[Proof of Property~\ref{prop:scale-invariance}]
$\rankingScore[k\randVarImportance]=\frac{\sum_{\aSample\in\sampleSpace}k\randVarImportance(\aSample)\randVarSatisfaction(\aSample)\aPerformance(\{\aSample\})}{\sum_{\aSample\in\sampleSpace}k\randVarImportance(\aSample)\aPerformance(\{\aSample\})}=\frac{\sum_{\aSample\in\sampleSpace}\randVarImportance(\aSample)\randVarSatisfaction(\aSample)\aPerformance(\{\aSample\})}{\sum_{\aSample\in\sampleSpace}\randVarImportance(\aSample)\aPerformance(\{\aSample\})}=\rankingScore[\randVarImportance]$
\end{proof}
\begin{proof}[Proof of Property~\ref{prop:scale-invariance-per-satisfaction}]
Let us consider a binary satisfaction, that is $\randVarSatisfaction(\aSample)\in\{0,1\}\,\forall\aSample\in\sampleSpace$.
Let us define the events $\anEvent_{0}=\{\aSample\in\sampleSpace:\randVarSatisfaction(\aSample)=0\}$
and $\anEvent_{1}=\{\aSample\in\sampleSpace:\randVarSatisfaction(\aSample)=1\}$.
If $\randVarImportance'=(\indicatorSymbol_{\randVarSatisfaction=0}\alpha_{0}+\indicatorSymbol_{\randVarSatisfaction=1}\alpha_{1})\randVarImportance$
with $\alpha_{0}>0$ and $\alpha_{1}>0$, then
\begin{align*}
\rankingScore[\randVarImportance](\aPerformance) & =\frac{\sum_{\aSample\in\sampleSpace}\randVarImportance(\aSample)\randVarSatisfaction(\aSample)\aPerformance(\{\aSample\})}{\sum_{\aSample\in\sampleSpace}\randVarImportance(\aSample)\aPerformance(\{\aSample\})}\\
 & =\frac{\sum_{\aSample\in\anEvent_{1}}\randVarImportance(\aSample)\aPerformance(\{\aSample\})}{\sum_{\aSample\in\anEvent_{0}}\randVarImportance(\aSample)\aPerformance(\{\aSample\})+\sum_{\aSample\in\anEvent_{1}}\randVarImportance(\aSample)\aPerformance(\{\aSample\})}
\end{align*}
and
\begin{align*}
\rankingScore[\randVarImportance'](\aPerformance) & =\frac{\sum_{\aSample\in\sampleSpace}\randVarImportance'(\aSample)\randVarSatisfaction(\aSample)\aPerformance(\{\aSample\})}{\sum_{\aSample\in\sampleSpace}\randVarImportance'(\aSample)\aPerformance(\{\aSample\})}\\
 & =\frac{\sum_{\aSample\in\anEvent_{1}}\randVarImportance'(\aSample)\aPerformance(\{\aSample\})}{\sum_{\aSample\in\anEvent_{0}}\randVarImportance'(\aSample)\aPerformance(\{\aSample\})+\sum_{\aSample\in\anEvent_{1}}\randVarImportance'(\aSample)\aPerformance(\{\aSample\})}\\
 & =\frac{\sum_{\aSample\in\anEvent_{1}}\alpha_{1}\randVarImportance(\aSample)\aPerformance(\{\aSample\})}{\sum_{\aSample\in\anEvent_{0}}\alpha_{0}\randVarImportance(\aSample)\aPerformance(\{\aSample\})+\sum_{\aSample\in\anEvent_{1}}\alpha_{1}\randVarImportance(\aSample)\aPerformance(\{\aSample\})}\\
 & =\frac{\alpha_{1}\sum_{\aSample\in\anEvent_{1}}\randVarImportance(\aSample)\aPerformance(\{\aSample\})}{\alpha_{0}\sum_{\aSample\in\anEvent_{0}}\randVarImportance(\aSample)\aPerformance(\{\aSample\})+\alpha_{1}\sum_{\aSample\in\anEvent_{1}}\randVarImportance(\aSample)\aPerformance(\{\aSample\})}
\end{align*}
Thus, $\rankingScore[\randVarImportance']=\frac{\alpha_{1}\rankingScore[\randVarImportance]}{\alpha_{0}(1-\rankingScore[\randVarImportance])\alpha_{1}\rankingScore[\randVarImportance]}$
and $\frac{\partial\rankingScore[\randVarImportance']}{\partial\rankingScore[\randVarImportance]}=\frac{\alpha_{0}\alpha_{1}}{\left(\alpha_{0}(1-\rankingScore[\randVarImportance])\alpha_{1}\rankingScore[\randVarImportance]\right)^{2}}>0$.
This leads immediately to the conclusion that $\ordering_{\rankingScore[\randVarImportance']}=\ordering_{\rankingScore[\randVarImportance]}$.
\end{proof}
\begin{proof}[Proof of Properties~\ref{prop:mean-vertical} and~\ref{prop:mean-horizontal}]
Let us consider a binary satisfaction and the events $\anEvent_{0}=\{\aSample\in\sampleSpace:\randVarSatisfaction(\aSample)=0\}$
and $\anEvent_{1}=\{\aSample\in\sampleSpace:\randVarSatisfaction(\aSample)=1\}$.
Let $\randVarImportance_{1}$ and $\randVarImportance_{2}$ be two
random variables and $\randVarImportance=\lambda_{1}\randVarImportance_{1}+\lambda_{2}\randVarImportance_{2}$
with $\lambda_{1},\lambda_{2}\in\realNumbers$ such that $\lambda_{1}+\lambda_{2}=1$.
\begin{itemize}
\item When the random variables $\randVarImportance_{1}$ and $\randVarImportance_{2}$
are such that $\randVarImportance_{1}(\aSample)=\randVarImportance_{2}(\aSample)=\randVarImportance(\aSample)\,\forall\aSample\in\anEvent_{1}$,
if we take $f:x\mapsto x^{-1}$,
\begin{align*}
 & \lambda_{1}f\left(\rankingScore[\randVarImportance_{1}](\aPerformance)\right)+\lambda_{2}f\left(\rankingScore[\randVarImportance_{1}](\aPerformance)\right)\\
= & \lambda_{1}\left(1+\frac{\sum_{\aSample\in\anEvent_{0}}\randVarImportance_{1}(\aSample)\aPerformance(\{\aSample\})}{\sum_{\aSample\in\anEvent_{1}}\randVarImportance_{1}(\aSample)\aPerformance(\{\aSample\})}\right)+\lambda_{2}\left(1+\frac{\sum_{\aSample\in\anEvent_{0}}\randVarImportance_{2}(\aSample)\aPerformance(\{\aSample\})}{\sum_{\aSample\in\anEvent_{1}}\randVarImportance_{2}(\aSample)\aPerformance(\{\aSample\})}\right)\\
= & \lambda_{1}\left(1+\frac{\sum_{\aSample\in\anEvent_{0}}\randVarImportance_{1}(\aSample)\aPerformance(\{\aSample\})}{\sum_{\aSample\in\anEvent_{1}}\randVarImportance(\aSample)\aPerformance(\{\aSample\})}\right)+\lambda_{2}\left(1+\frac{\sum_{\aSample\in\anEvent_{0}}\randVarImportance_{2}(\aSample)\aPerformance(\{\aSample\})}{\sum_{\aSample\in\anEvent_{1}}\randVarImportance(\aSample)\aPerformance(\{\aSample\})}\right)\\
= & (\lambda_{1}+\lambda_{2})+\frac{\lambda_{1}\sum_{\aSample\in\anEvent_{0}}\randVarImportance_{1}(\aSample)\aPerformance(\{\aSample\})+\lambda_{2}\sum_{\aSample\in\anEvent_{0}}\randVarImportance_{2}(\aSample)\aPerformance(\{\aSample\})}{\sum_{\aSample\in\anEvent_{1}}\randVarImportance(\aSample)\aPerformance(\{\aSample\})}\\
= & 1+\frac{\sum_{\aSample\in\anEvent_{0}}(\lambda_{1}\randVarImportance_{1}+\lambda_{2}\randVarImportance_{2})(\aSample)\aPerformance(\{\aSample\})}{\sum_{\aSample\in\anEvent_{1}}\randVarImportance(\aSample)\aPerformance(\{\aSample\})}\\
= & 1+\frac{\sum_{\aSample\in\anEvent_{0}}\randVarImportance(\aSample)\aPerformance(\{\aSample\})}{\sum_{\aSample\in\anEvent_{1}}\randVarImportance(\aSample)\aPerformance(\{\aSample\})}\\
= & f\left(\rankingScore[\randVarImportance](\aPerformance)\right)
\end{align*}
\item When the random variables $\randVarImportance_{1}$ and $\randVarImportance_{2}$
are such that $\randVarImportance_{1}(\aSample)=\randVarImportance_{2}(\aSample)=\randVarImportance(\aSample)\,\forall\aSample\in\anEvent_{0}$,
if we take $f:x\mapsto(1-x)^{-1}$,
\begin{align*}
 & \lambda_{1}f\left(\rankingScore[\randVarImportance_{1}](\aPerformance)\right)+\lambda_{2}f\left(\rankingScore[\randVarImportance_{1}](\aPerformance)\right)\\
= & \lambda_{1}\left(1+\frac{\sum_{\aSample\in\anEvent_{1}}\randVarImportance_{1}(\aSample)\aPerformance(\{\aSample\})}{\sum_{\aSample\in\anEvent_{0}}\randVarImportance_{1}(\aSample)\aPerformance(\{\aSample\})}\right)+\lambda_{2}\left(1+\frac{\sum_{\aSample\in\anEvent_{1}}\randVarImportance_{2}(\aSample)\aPerformance(\{\aSample\})}{\sum_{\aSample\in\anEvent_{0}}\randVarImportance_{2}(\aSample)\aPerformance(\{\aSample\})}\right)\\
= & \lambda_{1}\left(1+\frac{\sum_{\aSample\in\anEvent_{1}}\randVarImportance_{1}(\aSample)\aPerformance(\{\aSample\})}{\sum_{\aSample\in\anEvent_{0}}\randVarImportance(\aSample)\aPerformance(\{\aSample\})}\right)+\lambda_{2}\left(1+\frac{\sum_{\aSample\in\anEvent_{1}}\randVarImportance_{2}(\aSample)\aPerformance(\{\aSample\})}{\sum_{\aSample\in\anEvent_{0}}\randVarImportance(\aSample)\aPerformance(\{\aSample\})}\right)\\
= & (\lambda_{1}+\lambda_{2})+\frac{\lambda_{1}\sum_{\aSample\in\anEvent_{1}}\randVarImportance_{1}(\aSample)\aPerformance(\{\aSample\})+\lambda_{2}\sum_{\aSample\in\anEvent_{1}}\randVarImportance_{2}(\aSample)\aPerformance(\{\aSample\})}{\sum_{\aSample\in\anEvent_{0}}\randVarImportance(\aSample)\aPerformance(\{\aSample\})}\\
= & 1+\frac{\sum_{\aSample\in\anEvent_{1}}(\lambda_{1}\randVarImportance_{1}+\lambda_{2}\randVarImportance_{2})(\aSample)\aPerformance(\{\aSample\})}{\sum_{\aSample\in\anEvent_{0}}\randVarImportance(\aSample)\aPerformance(\{\aSample\})}\\
= & 1+\frac{\sum_{\aSample\in\anEvent_{1}}\randVarImportance(\aSample)\aPerformance(\{\aSample\})}{\sum_{\aSample\in\anEvent_{0}}\randVarImportance(\aSample)\aPerformance(\{\aSample\})}\\
= & f\left(\rankingScore[\randVarImportance](\aPerformance)\right)
\end{align*}
\end{itemize}
\end{proof}

\begin{proof}[Proof of Property~\ref{prop:convexity-contour-sets}]
Let $\aRelation \in \{<,\le,=,\ge,>\}$ and $v=\rankingScore(\aPerformance)$. For all $\aPerformance'\in\allPerformances$, we have:
\begin{align}
    & \rankingScore(\aPerformance') \aRelation \rankingScore(\aPerformance) \\
    \Leftrightarrow & \rankingScore(\aPerformance') \aRelation v \\
    \Leftrightarrow & \frac{\sum_{\aSample\in\sampleSpace}\randVarImportance(\aSample)\randVarSatisfaction(\aSample)\aPerformance'(\{\aSample\})}{\sum_{\aSample\in\sampleSpace}\randVarImportance(\aSample)\aPerformance'(\{\aSample\})} \aRelation v \\
    \Leftrightarrow & \left[\sum_{\aSample\in\sampleSpace}\randVarImportance(\aSample)\randVarSatisfaction(\aSample)\aPerformance'(\{\aSample\})\right] \aRelation  \left[v \sum_{\aSample\in\sampleSpace}\randVarImportance(\aSample)\aPerformance'(\{\aSample\})\right] \\
    \Leftrightarrow & \sum_{\aSample\in\sampleSpace}\randVarImportance(\aSample)\left[\randVarSatisfaction(\aSample)-v\right]\aPerformance'(\{\aSample\}) \aRelation 0
\end{align}
This is either a linear equality or a linear inequality constraint. Thus,
\begin{align}
    \phi_\aRelation(\aPerformance) & = \left\{\aPerformance'\in\allPerformances:\rankingScore(\aPerformance')\aRelation\rankingScore(\aPerformance)\right\} \\
    & = \left\{\aPerformance'\in\allPerformances:\sum_{\aSample\in\sampleSpace}\randVarImportance(\aSample)\left[\randVarSatisfaction(\aSample)-v\right]\aPerformance'(\{\aSample\}) \aRelation 0\right\}
\end{align}
is a convex subset of $\allPerformances$.
\end{proof}

\clearpage
\subsection{Supplementary Material about \cref{sec:review}}

\subsubsection{Link between Classical Formulation and Ours.}

\cref{fig:formulations} shows the connections between the classical formulation of the two-class classification task and our formulation, as explained in~\cref{sec:two-class-crisp-classification}.

\begin{figure}
\,\hfill{}%
\ovalbox{
    \begin{minipage}[c]{0.3\columnwidth}%
        \begin{center}
        \textbf{Classical formulation}
        \par\end{center}
        $$\allClasses=\{\classNeg,\classPos\}$$
        $$(y,\hat{y})\in \allClasses^2$$
        \begin{align*}
        \sampleTN&=(\classNeg,\classNeg) & \sampleFP&=(\classNeg,\classPos)\\
        \sampleFN&=(\classPos,\classNeg) & \sampleTP&=(\classPos,\classPos)
        \end{align*}
    \end{minipage}
}
\hfill{}%
\hfill{}%
\begin{minipage}[c]{0.3\columnwidth}%
    $$\sampleSpace=\allClasses^2$$
    $$\randVarSatisfaction=\indicatorSymbol_{\randVarGroundtruthClass=\randVarPredictedClass}$$
    $$\overrightarrow{\hspace{0.75\linewidth}}$$
    $$\overleftarrow{\hspace{0.75\linewidth}}$$
    $$\randVarGroundtruthClass:\sampleSpace\rightarrow \allClasses
    \qquad 
    \randVarPredictedClass:\sampleSpace\rightarrow \allClasses$$
    $$\aSample=(\randVarGroundtruthClass(\aSample),\randVarPredictedClass(\aSample)) \quad \forall \aSample\in\sampleSpace$$
\end{minipage}
\hfill{}%
\hfill{}%
\ovalbox{
    \begin{minipage}[c]{0.3\columnwidth}%
        \begin{center}
        \textbf{Our formulation}
        \par\end{center}
        $$\sampleSpace=\{\sampleTN,\sampleFP,\sampleFN,\sampleTP\}$$
        $$\eventSpace=2^\sampleSpace$$
        $$\randVarSatisfaction(\sampleFP)=\randVarSatisfaction(\sampleFN)=0$$
        $$\randVarSatisfaction(\sampleTN)=\randVarSatisfaction(\sampleTP)=1$$
    \end{minipage}
}
\hfill{}\,%

\caption{Passages between two formulations (left: classical, right: ours) for the performance analysis of two-class classification problems.}
\label{fig:formulations}
\end{figure}

\subsubsection{Custom Optimization Algorithm to Estimate Kendall's $\tau$.}
\label{sec:kendall}

For any score $\aScore$, our algorithm aims at determining the minimum and maximum values that a rank correlation between $\aScore$ and our ranking scores $\rankingScore$ can take over all possible importances $\randVarImportance$. Note that this algorithm is not specific to Kendall's $\tau$~\cite{Kendall1938ANewMeasure} and could also be used with any other rank correlation, for example Spearman's $\rho$~\cite{Spearman1904Theproof}.

\begin{description}
    
    \item[Variables.] Leveraging Properties~\ref{prop:scale-invariance} and~\ref{prop:scale-invariance-per-satisfaction}, we know that the rank-correlation between $\aScore$ and a ranking score $\rankingScore[\randVarImportance_1]$ is equal to the rank-correlation between $\aScore$ and another ranking score $\rankingScore[\randVarImportance_2]$ if 
        $\frac{
            \randVarImportance_1(\sampleTP)
        }{
            \randVarImportance_1(\sampleTN)
            +\randVarImportance_1(\sampleTP)
        }=\frac{
            \randVarImportance_2(\sampleTP)
        }{
            \randVarImportance_2(\sampleTN)
            +\randVarImportance_2(\sampleTP)
        }$
    and 
        $\frac{
            \randVarImportance_1(\sampleFN)
        }{
            \randVarImportance_1(\sampleFP)
            +\randVarImportance_1(\sampleFN)
        }=\frac{
            \randVarImportance_2(\sampleFN)
        }{
            \randVarImportance_2(\sampleFP)
            +\randVarImportance_2(\sampleFN)
        }$. 
    For this reason, we consider only two variables: 
        $a = \frac{
            \randVarImportance(\sampleTP)
        }{
            \randVarImportance(\sampleTN)
            +\randVarImportance(\sampleTP)
        }\in[0,1]$ 
    and
        $b = \frac{
            \randVarImportance(\sampleFN)
        }{
            \randVarImportance(\sampleFP)
            +\randVarImportance(\sampleFN)
        }\in[0,1]$. 
    
    \item[Objective function.] We optimize the function $\tau(a,b)$ that gives the rank correlation between $\aScore$ and $\rankingScore[\randVarImportance^*]$ with 
    $\randVarImportance^*(\sampleTP)=1-a$, 
    $\randVarImportance^*(\sampleTP)=1-b$, 
    $\randVarImportance^*(\sampleTP)=b$, and 
    $\randVarImportance^*(\sampleTP)=a$. In practice, this is an estimation based on a finite set of performances on which $\aScore$ and $\rankingScore[\randVarImportance^*]$ are applied. Note that $\tau(a,b)$ is not a continuous function when estimated on a finite set of performances. The chosen optimization technique circumvents the difficulties related to that.
    
    \item[Optimization technique.] We implemented a custom coarse-to-fine grid-based direct search~\cite{Conn2009Introduction}: we compute $\tau(a,b)$ on a coarse grid over the unit square, locate the maximum on the grid, center a smaller square and a finer grid around that point, and iterate until the square is small enough.
    
\end{description}

\subsubsection{Scores Perfectly Correlated with a Ranking Score, for all Performances}
\label{sec:perfect-correlation-proofs-1}

\begin{itemize}

\item \textbf{The accuracy:}
    $\scoreAccuracy=\rankingScore$ with 
    $\randVarImportance(\sampleTN)=\nicefrac12$,
    $\randVarImportance(\sampleFP)=\nicefrac12$,
    $\randVarImportance(\sampleFN)=\nicefrac12$,
    and $\randVarImportance(\sampleTP)=\nicefrac12$.

\item \textbf{The F-score with $\beta=0.5$:}
    $\scoreFBeta[0.5]=\rankingScore$ with 
    $\randVarImportance(\sampleTN)=0$,
    $\randVarImportance(\sampleFP)=\nicefrac45$,
    $\randVarImportance(\sampleFN)=\nicefrac15$,
    and $\randVarImportance(\sampleTP)=1$.

\item \textbf{The F-score with $\beta=1.0$:}
    $\scoreFBeta[1]=\rankingScore$ with 
    $\randVarImportance(\sampleTN)=0$,
    $\randVarImportance(\sampleFP)=\nicefrac12$,
    $\randVarImportance(\sampleFN)=\nicefrac12$,
    and $\randVarImportance(\sampleTP)=1$.

\item \textbf{The F-score with $\beta=2.0$:}
    $\scoreFBeta[2]=\rankingScore$ with 
    $\randVarImportance(\sampleTN)=0$,
    $\randVarImportance(\sampleFP)=\nicefrac15$,
    $\randVarImportance(\sampleFN)=\nicefrac45$,
    and $\randVarImportance(\sampleTP)=1$.

\item \textbf{The negative predictive value:}
    $\scoreNPV=\rankingScore$ with 
    $\randVarImportance(\sampleTN)=1$,
    $\randVarImportance(\sampleFP)=0$,
    $\randVarImportance(\sampleFN)=1$,
    and $\randVarImportance(\sampleTP)=0$.

\item \textbf{The positive predictive value:}
    $\scorePPV=\rankingScore$ with 
    $\randVarImportance(\sampleTN)=0$,
    $\randVarImportance(\sampleFP)=1$,
    $\randVarImportance(\sampleFN)=0$,
    and $\randVarImportance(\sampleTP)=1$.

\item \textbf{The true negative rate:}
    $\scoreTNR=\rankingScore$ with 
    $\randVarImportance(\sampleTN)=1$,
    $\randVarImportance(\sampleFP)=1$,
    $\randVarImportance(\sampleFN)=0$,
    and $\randVarImportance(\sampleTP)=0$.

\item \textbf{The true positive rate:}
    $\scoreTPR=\rankingScore$ with 
    $\randVarImportance(\sampleTN)=0$,
    $\randVarImportance(\sampleFP)=0$,
    $\randVarImportance(\sampleFN)=1$,
    and $\randVarImportance(\sampleTP)=1$.
    
\end{itemize}

\subsubsection{Scores Perfectly Correlated with a Ranking Score, for the Performances Corresponding to Given Class Priors $\priorneg\ne0$ and $\priorpos\ne0$}
\label{sec:perfect-correlation-proofs-2}

\begin{itemize}

\item \textbf{The balanced accuracy:}
    $\scoreBalancedAccuracy=\rankingScore$ with 
    $\randVarImportance(\sampleTN)=\priorpos$,
    $\randVarImportance(\sampleFP)=\priorpos$,
    $\randVarImportance(\sampleFN)=\priorneg$,
    and $\randVarImportance(\sampleTP)=\priorneg$.
    
\item \textbf{Cohen's kappa:}
    $\scoreCohenKappa=\frac{\rankingScore-2\priorneg\priorpos}{\priorneg^2+\priorpos^2}$ with 
    $\randVarImportance(\sampleTN)=\frac{\priorpos^2}{\priorneg^2+\priorpos^2}$,
    $\randVarImportance(\sampleFP)=\frac12$,
    $\randVarImportance(\sampleFN)=\frac12$,
    and $\randVarImportance(\sampleTP)=\frac{\priorneg^2}{\priorneg^2+\priorpos^2}$. Thus, $\frac{\partial\scoreCohenKappa}{\rankingScore}>0$.
    
\item \textbf{The informedness (\aka Youden's J):}
    $\scoreYoudenJ=2\rankingScore-1$ with 
    $\randVarImportance(\sampleTN)=\priorpos$,
    $\randVarImportance(\sampleFP)=\priorpos$,
    $\randVarImportance(\sampleFN)=\priorneg$,
    and $\randVarImportance(\sampleTP)=\priorneg$. Thus, $\frac{\partial\scoreYoudenJ}{\rankingScore}>0$.
    
\item \textbf{The negative likelihood ratio:}
    $\scoreNLR=\frac{1-\rankingScore}{\rankingScore}$ with 
    $\randVarImportance(\sampleTN)=1$,
    $\randVarImportance(\sampleFP)=0$,
    $\randVarImportance(\sampleFN)=1$,
    and $\randVarImportance(\sampleTP)=0$. Thus, $\frac{\partial\scoreNLR}{\rankingScore}<0$.
    
\item \textbf{The positive likelihood ratio:}
    $\scorePLR=\frac{\rankingScore}{1-\rankingScore}$ with 
    $\randVarImportance(\sampleTN)=0$,
    $\randVarImportance(\sampleFP)=1$,
    $\randVarImportance(\sampleFN)=0$,
    and $\randVarImportance(\sampleTP)=1$. Thus, $\frac{\partial\scorePLR}{\rankingScore}>0$.
    
\item \textbf{The probability of the elementary event \emph{true negative}:}
    $\scorePTN=\priorneg\rankingScore$ with 
    $\randVarImportance(\sampleTN)=1$,
    $\randVarImportance(\sampleFP)=1$,
    $\randVarImportance(\sampleFN)=0$,
    and $\randVarImportance(\sampleTP)=0$. Thus, $\frac{\partial\scorePTN}{\rankingScore}>0$.
    
\item \textbf{The probability of the elementary event \emph{true positive}:}
    $\scorePTP=\priorpos\rankingScore$ with 
    $\randVarImportance(\sampleTN)=0$,
    $\randVarImportance(\sampleFP)=0$,
    $\randVarImportance(\sampleFN)=1$,
    and $\randVarImportance(\sampleTP)=1$. Thus, $\frac{\partial\scorePTP}{\rankingScore}>0$.
    
\end{itemize}

\stopcontents[mytoc]

\end{document}